\documentclass{article}
\usepackage{amsmath,amssymb,amsthm}
\usepackage{stmaryrd,booktabs}
\usepackage{commath}
\usepackage{fullpage}
\usepackage{multirow,rotating}
\usepackage{url}
\usepackage{natbib}
\usepackage{color}
\usepackage{graphicx}

\usepackage{algorithm,algorithmic}
\usepackage[colorlinks=true,linkcolor=red]{hyperref}

\title{Surrogate regret bounds for generalized classification performance metrics\thanks{W. Kot{\l}owski has been supported by the Polish National Science Centre under grant no. 2013/11/D/ST6/03050. K. Dembczy{\'n}ski has been supported by the Polish National Science Centre under grant no. 2013/09/D/ST6/03917.}
}

\author{Wojciech Kot{\l}owski \\ \small{Poznan University of Technology, Poland} \and Krzysztof Dembczy{\'n}ski \\ \small{Poznan University of Technology, Poland}}

\newtheorem{theorem}{Theorem}
\newtheorem{lemma}[theorem]{Lemma}

\newtheorem{proposition}[theorem]{Proposition}

\theoremstyle{definition}

\newcommand{\assert}[1]{\llbracket #1 \rrbracket}
\newcommand{\sgn}{\mathrm{sgn}}
\newcommand{\regret}{\mathrm{Reg}}
\newcommand{\regretc}{\mathrm{reg}}
\newcommand{\risk}{\mathrm{Risk}}
\newcommand{\riskc}{\mathrm{risk}}

\newcommand{\TP}{\mathrm{TP}}
\newcommand{\FP}{\mathrm{FP}}
\newcommand{\FN}{\mathrm{FN}}
\newcommand{\TN}{\mathrm{TN}}

\newcommand{\vecy}{\boldsymbol{y}}
\newcommand{\vech}{\boldsymbol{h}}
\newcommand{\argmin}{\operatornamewithlimits{argmin}}
\newcommand{\argmax}{\operatornamewithlimits{argmax}}

\begin{document}

\maketitle

\begin{abstract} 
We consider optimization of generalized performance metrics for binary classification
by means of surrogate losses.  We focus on a class of metrics, which are linear-fractional functions
of the false positive and false negative rates (examples of which include $F_{\beta}$-measure,
Jaccard similarity coefficient, AM measure, and many others).
Our analysis concerns the following two-step procedure. First, a real-valued function
$f$ is learned by minimizing a surrogate loss for binary classification on the training sample.
It is assumed that the surrogate loss
is a strongly proper composite loss function (examples of which include
logistic loss, squared-error loss, exponential loss, etc.). Then, given $f$, a threshold $\widehat{\theta}$ is
tuned on a separate validation sample, by direct optimization of the target performance metric.
We show that the regret of the resulting classifier (obtained from thresholding $f$ on $\widehat{\theta}$) 
measured with respect to the target metric is upperbounded by the regret of $f$ measured with respect
to the surrogate loss. 
We also extend our results to cover multilabel classification and provide regret bounds
for micro- and macro-averaging measures. 
Our findings are further analyzed
in~a~computational study on both synthetic and real data~sets.
%\keywords{Generalized performance metric \and regret bound \and surrogate loss function \and binary classification \and multilabel classification \and F-measure \and Jaccard similarity \and AM measure}
\end{abstract}

\section{Introduction}
\label{sec:introduction}

In binary classification, misclassification error is 
not necessarily an adequate evaluation metric,
and one often resorts to more complex metrics, 
better suited for the problem. For instance,
when classes are imbalanced, 
$F_{\beta}$-measure \citep{lewis95,Jansche_2005,Ye2012}
and AM measure (balanced error rate) \citep{Menon_etal2013} are frequently
used. 
Optimizing such generalized performance
metrics poses computational and statistical challenges,
as they cannot be decomposed into losses on individual
observations. 

In this paper, 
we consider optimization of generalized performance metrics
by means of surrogate losses. 
We restrict our attention to a family of performance metrics
which are ratios of linear functions of false positives (FP) 
and false negatives (FN). Such functions are called
linear-fractional, and include the aforementioned
$F_{\beta}$ and AM measures, as well as
Jaccard similarity coefficient, weighted accuracy, and many others \citep{Natarajan_etal2014,Koyejo15}.
We focus on the most popular
approach to optimizing generalized performance
metrics in practice, based on the following
two-step procedure. First, a real-valued function
$f$ is learned by minimizing a surrogate loss for 
binary classification on the training sample.
Then, given $f$, a threshold $\widehat{\theta}$ is
tuned on a separate validation sample, 
by direct optimization of the target performance measure
with respect to a classifier obtained 
from $f$ by thresholding at $\widehat{\theta}$, classifying all observations with value of $f$ above the threshold
as positive class, and all observations below the threshold as negative class.
This approach can be motivated by the asymptotic analysis:
minimization of appropriate surrogate loss
results in estimation of conditional (``posterior'') class probabilities,
and many performance metrics are maximized by a classifier
which predicts by thresholding on the scale 
of conditional probabilities \citep{Ye2012,Zhao_etal_2013,Natarajan_etal2014}.
However, it is unclear what can be said about the behavior of this procedure on finite samples.

In this work, we are interested in theoretical analysis and justification
of this approach for any sample size, and for
any, not necessarily perfect, classification function.
To this end, we use the notion of \emph{regret} with respect
to some evaluation metric, which is a difference between the performance of a given classifier
and the performance of the optimal classifier with respect to this metric.
We show that the regret of the resulting classifier (obtained from thresholding $f$ on $\widehat{\theta}$) 
measured with respect to the target metric is upperbounded by the regret of $f$ measured with respect
to the surrogate loss. Our result holds for any surrogate loss function, which is
\emph{strongly proper composite} \citep{Agarwal2014}, examples of which include
logistic loss, squared-error loss, exponential loss, etc.
Interestingly, the proof of our result goes by an intermediate bound of the regret with respect to
the target measure by
a cost-sensitive classification regret.
As a byproduct, we get a bound on the cost-sensitive classification regret 
by a surrogate regret of a real-valued function which holds \emph{simultaneously} for \emph{all} misclassification costs: 
the misclassification costs only influence the threshold, but not: the function, the surrogate loss, or the regret bound. 

We further extend our results to cover multilabel classification, in which the goal is to simultaneously
predict multiple labels for each object. We consider two methods of generalizing binary classification
performance metrics to the multilabel setting: 
the macro-averaging and the micro-averaging \citep{retrieval,Puthiya_etal2014,Koyejo15}. The macro-averaging is based on
first computing the performance metric separately for each label, and then
averaging the metrics over the labels. In the micro-averaging, the false positives
and false negatives for each label are first averaged over the labels, and then the performance metric is calculated on these averaged
quantities.
We show that our regret bounds hold for both macro- and micro-averaging measures. Interestingly, for micro averaging,
only a single threshold needs to be tuned and is shared among all labels.

Our finding is further analyzed in a computational study on both synthetic and real data sets. We compare
the performance of the algorithm when used with two types of surrogate losses: the logistic loss (which is strongly proper)
and the hinge loss (which is not a proper loss).
On synthetic data sets, we analyze the behavior of the algorithm for discrete feature distribution (where nonparametric
classifiers are used), and for continuous feature distribution (where linear classifiers are used).
 Next, we look at the performance of the algorithm on the real-life benchmark data sets, both for binary and multilabel classification.

We note that the goal of this paper is not to propose a new learning algorithm,
but rather to provide a deeper statistical understanding of an existing method. 
The two-step procedure analyzed here (also known as the plug-in method in the case when the outcomes of the function have a probabilistic interpretation),
is commonly used in the binary classification with generalized performance metrics, but 
this is exactly the reason why we think it is important to study this method in more depth from a theoretical point of view. 
%We hope our analysis provides more insight into why the two-stage procedure method works. 
%Existing theoretical work is mainly concerned with statistical consistency/calibration, which is only an asymptotic property. Here, for the first time we give regret bounds, valid for all finite samples and informing about the rate of convergence.

\subsection{Related work}

In machine learning, numerous attempts to optimize generalized performance metrics
have been proposed. They can be divided
into two general categories. The structured loss approaches
\citep{Musicant2012,TsochantaridisJHA05,PettersonSMLL,PettersonRML}
rely on incorporating the performance metric into the
training process, thus requiring specialized learning algorithms
to optimize non-standard objectives. On the other hand, the plug-in approaches, 
which are very closely related to the topic of this work,
are based on obtaining
reliable class conditional probability estimates by employing
standard algorithms minimizing some surrogate
loss for binary classification (such as logistic loss used in logistic regression, exponential loss used in boosting, etc.),
and then plugging these estimates into the functional form of
the optimal prediction rule for a given performance metric
\citep{Jansche_2007,Ye2012,Dembczynski_et_al_2013,Waegeman2013,Narasimhan_etal2014,Narasimhan15,Natarajan_etal2014,Koyejo15}.

Existing theoretical work on generalized performance metrics
is mainly concerned with \emph{statistical consistency} also
known as \emph{calibration}, which determines whether convergence to the minimizer of a 
surrogate loss implies convergence to the minimizer of the task performance measure as the sample size
goes to infinity \citep{Dembczynski_et_al_2010a,Ye2012,Gao_Zhou_2013,Zhao_etal_2013,Narasimhan_etal2014,Natarajan_etal2014,Koyejo15}.
Here we give a stronger result which bounds the regret with respect
to the performance metric by the regret with respect to the surrogate loss. Our result is valid for all
finite sample sizes and informs about the rates of convergence.

We also note that two distinct frameworks are used to study the statistical consistency of classifiers
with respect to performance metrics: Decision Theoretic Analysis (DTA), which assumes a test set of a fixed size,
and Empirical Utility Maximization (EUM), in which the metric is defined by means
of population quantities \citep{Ye2012}. In this context, our work falls into the EUM framework.

\citet{Puthiya_etal2014} presented an alternative approach to 
maximizing linear-fractional metrics
by learning a sequence of binary classification problems 
with varying misclassification costs. 
While we were inspired by
their theoretical analysis,
their 
approach is, however, more complicated than 
the two-step approach analyzed here, which requires
solving an ordinary binary classification problem only once.
Moreover, as part of our proof, we show 
that by minimizing a strongly proper composite loss, we are \emph{implicitly}
minimizing cost-sensitive classification error for any
misclassification costs without any overhead. Hence,
the costs need not be known during
learning, and can only be determined later
on a separate validation sample
by optimizing the threshold.
\citet{Narasimhan15} developed a general framework for designing provably consistent algorithms for complex multiclass
performance measures. They relate the regret with respect to the target metric to the conditional probability
estimation error measured in terms of $L_1$-metric. Their algorithms rely on using accurate class conditional probability estimates and
multiple solving cost-sensitive multiclass classification problems.

The generalized performance metrics for binary classification are employed in the multilabel setting by means
of one of the three averaging schemes \citep{Waegeman2013,Puthiya_etal2014,Koyejo15}: 
instance-averaging (averaging errors over the labels, averaging metric over the examples),
macro-averaging (averaging errors over the examples, averaging metric over the labels),
and micro-averaging (averaging errors over the examples and the labels).
\citet{Koyejo15} characterize the optimal classifiers for multilabel
metrics %in the EUM framework
and prove the consistency of the plug-in method. Our regret bounds for multilabel
classification can be seen as a follow up on their work.

\subsection{Outline} The paper is organized as follows.
In Section \ref{sec:problem_setting} 
we introduce basic
concepts, definitions and notation. The main result
is presented in Section \ref{sec:main_result} and
proved in Section \ref{sec:proof}. Section \ref{sec:multilabel}
extends our results to the multilabel setting. The theoretical
contribution of the paper is complemented by 
computational experiments in Section \ref{sec:experiment},
prior to concluding with a summary
in Section \ref{sec:summary}.

\section{Problem setting}
\label{sec:problem_setting}

\subsection{Binary classification}

%\paragraph{Binary classifier.}
In binary classification, the goal is, given an input (feature vector) $x \in X$,
to accurately predict the output (label) $y \in \{-1,1\}$. We assume input-output
pairs $(x,y)$ are generated i.i.d. according to $\Pr(x,y)$.
A \emph{classifier} is a mapping  $h \colon X \to \{-1,1\}$. Given $h$, we define
the following four quantities:
\begin{align*}
\TP(h) &= \Pr(h(x) = 1 \land y=1), \\
\FP(h) &= \Pr(h(x) = 1 \land y=-1),\\
\TN(h) &= \Pr(h(x) = -1 \land y=-1),\\
\FN(h) &= \Pr(h(x) = -1 \land y=1),
\end{align*}
which are known as \emph{true positives}, \emph{false positives}, \emph{true negatives} and 
\emph{false negatives}, respectively. We also denote $\Pr(y=1)$ by $P$.
Note that for any $h$, $\FP(h) + \TN(h) = \Pr(y=-1) = 1-P$ and
$\TP(h) + \FN(h) = P$, so out of the four quantities above, only two are independent. 
In this paper, we use the convention to parameterize all metrics by means of $\FP(h)$ and $\FN(h)$.

%\paragraph{Generalized classification performance metrics.}
We call a two-argument function $\Psi = \Psi(\FP,\FN)$ a 
\emph{(generalized) classification performance metric}. 
Given a classifier $h$, we define $\Psi(h) = \Psi(\FP(h),\FN(h))$. 
Throughout the paper we assume that $\Psi(\FP,\FN)$ is \emph{linear-fractional},
i.e., is a ratio of linear functions:
\begin{equation}
\label{eq:pseudo_linear_function}
 \Psi(\FP,\FN) = \frac{a_0 + a_1 \FP + a_2 \FN}{b_0 + b_1 \FP + b_2 \FN},
\end{equation}
where we allow coefficients $a_i,b_i$ to depend on the distribution $\Pr(x,y)$.
Note, that our convention to parameterize the metric by means of $(\FP,\FN)$
does not affect definition (\ref{eq:pseudo_linear_function}), because $\Psi$
can be reparameterized to $(\FP,\TN)$, $(\TP,\FN)$, or $(\TP,\TN)$, and
will remain linear-fractional in all these parameterizations. We also assume
$\Psi(\FP,\FN)$ is non-increasing in $\FP$ and $\FN$, a property
that is inherently possessed by virtually all performance measures used in practice.
Table \ref{tbl:performance_metrics} lists some popular examples of 
linear-fractional performance metrics.

\begin{table}
\begin{center}
\begin{tabular}{l@{~~~~~~~}l}
\toprule
metric & expression \\
\midrule
Accuracy & $\mathrm{Acc} = 1 - \FN - \FP$ \\[3mm]
$F_{\beta}$-measure & $F_{\beta} = \frac{(1+\beta^2) (P - \FN)}{(1+\beta^2)P - \FN + \FP}$\\[3mm]
Jaccard similarity & $J = \frac{P - \FN}{P + \FP}$\\[3mm]
AM measure & $\mathrm{AM} = \frac{2P(1-P) - P \,\FP - (1-P)\FN}{2P(1-P)}$\\[3mm]
Weighted accuracy & $\mathrm{WA} = \frac{w_1(1-P) + w_2 P - w_1 \FP - w_2 \FN}{w_1(1-P) + w_2 P}$\\[3mm]
\bottomrule
\end{tabular}
\end{center}
\caption{Some popular linear-fractional performance metrics expressed as functions of $\FN$ and $\FP$. 
See \citep{Natarajan_etal2014}
for a more detailed description.}
\label{tbl:performance_metrics}
\end{table}

Let $h^*_{\Psi}$ be the maximizer of $\Psi(h)$ over all classifiers:
\[
 h^*_{\Psi} = \argmax_{h \colon X \to \{-1,1\}} \Psi(h)
\]
(if $\argmax$ is not unique, we take $h^*_{\Psi}$ to be any maximizer of $\Psi$).
Given any classifier $h$, we define its \emph{$\Psi$-regret} as:% a distance of $h$ from the optimal
 %$h^*_{\Psi}$ measured by means of $\Psi$:
\[
\regret_{\Psi}(h) = \Psi(h^*_{\Psi}) - \Psi(h). 
\]
The $\Psi$-regret is nonnegative from the definition, and quantifies the suboptimality of $h$, i.e., how much worse is $h$ comparing to the optimal $h^*_{\Psi}$.
%The $\Psi$-regret is a better performance metric than $\Psi$ itself: when the data is hard and $\Psi(h^*_{\Psi})$ is small,
%all classifiers will result in small $\Psi$

\subsection{Strongly proper composite losses}

%\paragraph{Strongly proper composite losses.}
Here we briefly outline the theory of strongly proper composite loss functions.
See \citep{Agarwal2014} for a more detailed description.

Define a \emph{binary class probability estimation (CPE) loss function} 
\citep{Reid_Williamson_2010,Reid_Williamson_2011} as a function
$c \colon \{-1,1\} \times [0,1] \to \mathbb{R}_+$, where $c(y,\widehat{\eta})$
assigns penalty to prediction $\widehat{\eta}$, when the observed label is $y$.
Define the \emph{conditional $c$-risk} as:\footnote{Throughout the paper,
we follow the convention that all conditional
quantities are lowercase (regret, risk), while all unconditional quantities
are uppercase (Regret, Risk).}
\[
 \riskc_c(\eta,\widehat{\eta}) = \eta c(1,\widehat{\eta}) + (1-\eta) c(-1,\widehat{\eta}),
\]
the expected loss of prediction $\widehat{\eta}$ when the label is drawn from a distribution
with $\Pr(y=1) = \eta$. We say CPE loss is \emph{proper} if for any $\eta \in [0,1]$,
$\eta \in \argmin_{\widehat{\eta} \in [0,1]} \riskc_c(\eta,\widehat{\eta})$. In other words,
proper losses are minimized by taking the true class probability distribution as a prediction;
hence $\widehat{\eta}$ can be interpreted as probability estimate of $\eta$. Define the
\emph{conditional $c$-regret} as:
\begin{align*}
 \regretc_{c}(\eta,\widehat{\eta})
 &= \riskc_c(\eta,\widehat{\eta}) - \inf_{\widehat{\eta}'} \riskc_c(\eta,\widehat{\eta}') \\
 &= \riskc_c(\eta,\widehat{\eta}) - \riskc_{c}(\eta,\eta),
\end{align*}
the difference between the conditional $c$-risk of $\widehat{\eta}$ and the
optimal $c$-risk. We say a CPE loss $c$ is \emph{$\lambda$-strongly proper}
if for any $\eta, \widehat{\eta}$:
\[
 \regretc_{c}(\eta,\widehat{\eta}) \geq \frac{\lambda}{2} (\eta - \widehat{\eta})^2,
\]
i.e. the conditional $c$-regret is everywhere lowerbounded by a squared difference
of its arguments. It can be shown \citep{Agarwal2014} that under mild regularity assumption
a proper CPE loss $c$ is $\lambda$-strongly
proper if and only if the function $H_{c}(\eta) := \riskc_{c}(\eta,\eta)$ is $\lambda$-strongly
concave. This fact lets us easily verify whether a given loss function is $\lambda$-strongly proper.

It is often more convenient to reparameterize the loss function from $\widehat{\eta} \in [0,1]$
to a real-valued $f \in \mathbb{R}$ through a strictly increasing (and therefore invertible)
\emph{link function} $\psi \colon [0,1] \to \mathbb{R}$:
\[
\ell(y,f) = c\big(y,\psi^{-1}(f)\big).
\]
If $c$ is $\lambda$-strongly proper, we call function
$\ell \colon \{-1,1\} \times \mathbb{R} \to \mathbb{R}_+$ 
\emph{$\lambda$-strongly proper composite loss function}.
The notions of conditional $\ell$-risk $\riskc_\ell(\eta,f)$ 
and conditional $\ell$-regret $\regretc_{\ell}(\eta,f)$ extend naturally
to the case of composite losses:
\begin{align*}
 \riskc_\ell(\eta,f) &= \eta \ell(1,f) + (1-\eta) \ell(-1,f) \\[2mm]
 \regretc_{\ell}(\eta,f) &= \riskc_\ell(\eta,f) - \inf_{f' \in \mathbb{R}} \riskc_\ell(\eta,f')\\
 &= \riskc_\ell(\eta,f) - \riskc_\ell(\eta,\psi(\eta)).
\end{align*}
and the strong properness of underlying CPE loss implies:
\begin{equation}
 \regretc_{\ell}(\eta,f) \geq \frac{\lambda}{2} \left(\eta - \psi^{-1}(f) \right)^2
 \label{eq:strong_properness_of_ell}
\end{equation}
As an example, consider a \emph{logarithmic scoring rule}:
\[
c(y,\widehat{\eta}) = -\assert{y=1} \log \widehat{\eta} - \assert{y=-1} \log (1- \widehat{\eta}), 
\]
where $\assert{Q}$ is the indicator function, equal to $1$ if $Q$ holds, and to $0$ otherwise.
Its conditional risk is given by:
\[
\riskc_c(\eta,\widehat{\eta}) = -\eta \log \widehat{\eta} - (1-\eta) \log (1- \widehat{\eta}),
\]
the \emph{cross-entropy} between $\eta$ and $\widehat{\eta}$. The conditional $c$-regret
is the binary \emph{Kullback-Leibler divergence} between $\eta$ and $\widehat{\eta}$:
\[
 \regretc_{c}(\eta,\widehat{\eta}) = \eta \log \frac{\eta}{\widehat{\eta}} + (1-\eta) \log \frac{1-\eta}{1- \widehat{\eta}}.
\]
Note that since $H(\eta) = \riskc_c(\eta,\eta)$ is the binary entropy function, 
and $\big| \frac{\dif^2 H}{\dif \eta^2} \big| = \frac{1}{\eta(1-\eta)} \geq 4$,
$c$ is $4$-strongly proper loss.
Using the \emph{logit} link function $\psi(\widehat{\eta}) = \log \frac{\widehat{\eta}}{1-\widehat{\eta}}$,
we end up with the logistic loss function:
\[
 \ell(y,f) = \log\left(1 + e^{-yf} \right),
\]
which is $4$-strongly proper composite from the definition.

\begin{table}
\begin{center}
\begin{tabular}{c|ccc}
\toprule
loss function & squared-error & logistic & exponential \\[3mm]
$\ell(y,f)$ & $(y-f)^2$ & $\log\left(1+e^{-fy}\right)$ & $e^{-yf}$ \\[3mm]
$c(1,\widehat{\eta})$ & $4(1-\widehat{\eta})^2$ & $-\log \widehat{\eta}$ & $\sqrt{\frac{1-\widehat{\eta}}{\widehat{\eta}}}$ \\[3mm]
$c(-1,\widehat{\eta})$ & $4\widehat{\eta}^2$ & $-\log (1-\widehat{\eta})$ & $\sqrt{\frac{\widehat{\eta}}{1-\widehat{\eta}}}$ \\[3mm]
$\psi(\widehat{\eta})$ & $2 \widehat{\eta}-1$ & $\log \frac{\widehat{\eta}}{1-\widehat{\eta}}$ & $\frac{1}{2} \log \frac{\widehat{\eta}}{1-\widehat{\eta}}$ \\[3mm]
$\lambda$ & 8 & 4 & 4 \\
\bottomrule
\end{tabular}
\end{center}
\caption{Three popular strongly proper composite losses: squared-error, logistic and exponential losses. Shown are 
the formula $\ell(y,f)$, the underlying CPE loss $c(y,\widehat{\eta})$ with the link function $\psi(\widehat{\eta})$,
as well as the strong properness constant $\lambda$. See \citep{Agarwal2014} for more details and examples.}
\label{tbl:strongly_proper_composite_losses}
\end{table}

Table \ref{tbl:strongly_proper_composite_losses} presents some of the commonly used losses
which are strongly proper composite.
Note that the \emph{hinge loss} $\ell(y,f) = (1-yf)_+$, used, e.g., in support vector machines
\citep{FriedmanHastieTibshirani03},
 is \emph{not} strongly proper composite (even not proper composite).

\section{Main result}
\label{sec:main_result}

Given a real-valued function $f \colon X \to \mathbb{R}$, and a $\lambda$-strongly
proper composite loss $\ell(y,f)$, define the \emph{$\ell$-risk} of $f$ as
the expected loss of $f(x)$ with respect to the data distribution:
\begin{align*}
 \risk_\ell(f) &= \mathbb{E}_{(x,y)} \left[\ell(y,f(x)) \right] \\
 &= \mathbb{E}_{x} \left[\riskc_\ell(\eta(x),f(x)) \right], 
\end{align*}
where $\eta(x) = \Pr(y=1|x)$. Let $f_{\ell}^*$ be the minimizer
$\risk_\ell(f)$ over all functions, $f_{\ell}^* = \argmin_f \risk_\ell(f)$. 
Since $\ell$ is proper composite:
\[
 f_{\ell}^*(x) = \psi\big(\eta(x)\big).
\]
Define the $\ell$-regret of $f$ as:
\begin{align*}
 \regret_{\ell}(f) &= \risk_\ell(f) - \risk_\ell(f^*_{\ell}) \\
 &= \mathbb{E}_{x} \left[\riskc_\ell(\eta(x),f(x)) - \riskc_\ell(\eta(x), 
 f^*_{\ell}(x)) \right].
\end{align*}

Any real-valued function $f \colon X \to \mathbb{R}$
can be turned into a classifier $h_{f,\theta} \colon X \to \{-1,1\}$, by thresholding at
some value $\theta$:
\[
 h_{f,\theta}(x) = \sgn(f(x) - \theta).
\]
The purpose of this paper is to address the following problem: given a function $f$
with $\ell$-regret $\regret_{\ell}(f)$, and a threshold $\theta$, what can we say about
$\Psi$-regret of $h_{f,\theta}$? For instance, can we bound $\regret_{\Psi}(h_{f,\theta})$
in terms of $\regret_{\ell}(f)$? 
We give a positive answer to this question, which is based on the following regret bound:
\begin{lemma}
\label{lemma:main}
Let $\Psi(\FP,\FN)$ be a linear-fractional function of the form (\ref{eq:pseudo_linear_function}),
which is non-increasing in $\FP$ and $\FN$. Assume that there exists $\gamma > 0$,
such that for any classifier $h \colon X \to \{-1,1\}$:
\[
b_0 + b_1 \FP(h) + b_2 \FN(h) \geq \gamma,
\]
i.e. the denominator of $\Psi$ is positive and bounded away from zero.
Let $\ell$ be a $\lambda$-strongly proper composite loss function.
Then, there exists a threshold $\theta^*$,
such that for any real-valued function $f \colon X \to \mathbb{R}$,
\[
\regret_{\Psi}(h_{{f,\theta^*}})
 \leq C \sqrt{\frac{2}{\lambda}} \sqrt{\regret_{\ell}(f)},
\]
where $C = \frac{1}{\gamma}\left(\Psi(h_{\Psi}^*)(b_1 + b_2) - (a_1 + a_2)\right) > 0$.
\end{lemma}
The proof is quite long and hence is postponed to Section \ref{sec:proof}.
Interestingly, the proof goes by an intermediate bound of the $\Psi$-regret by
a cost-sensitive classification regret. We note that the bound in Lemma \ref{lemma:main}
is in general unimprovable, in the sense that it is easy to find $f$, $\Psi$, $\ell$, and
distribution $\Pr(x,y)$, for which the bound holds with equality (see proof for details).
We split the constant in front of the bound into $C$ and $\lambda$, because
$C$ depends only on $\Psi$, while $\lambda$ depends only on $\ell$. 
Table \ref{tbl:bounds} lists these constants
for some popular metrics. We note that constant $\gamma$ (lower bound on the denominator of $\Psi$)
will be \emph{distribution-dependent}
in general (as it can depend on $P=\Pr(y=1)$) and may not have a uniform lower bound which holds
for all distributions.

Lemma \ref{lemma:main} has the following interpretation. 
If we are able to find a function $f$ with small $\ell$-regret, 
we are guaranteed that there exists a threshold $\theta^*$ such that $h_{f,\theta^*}$ has small $\Psi$-regret. 
%(which depends on the performance metric, surrogate loss and the data distribution, but does not depend on the function $f$)
Note that the same threshold $\theta^*$ will work for any $f$, 
and the right hand side of the bound is \emph{independent} of $\theta^*$.
Hence, to minimize the right hand side we only need to minimize 
$\ell$-regret, and we can deal with the threshold afterwards.

Lemma \ref{lemma:main} also reveals the form of
the optimal classifier $h_{\Psi}^*$: take $f=f^*_{\ell}$
in the lemma and note that  $\regret_{\ell}(f^*_{\ell})=0$,
so that $\regret_{\Psi}(h_{f^*_{\ell},\theta^*}) = 0$, which means
that $h_{f^*_{\ell},\theta^*}$ is the minimizer of $\Psi$:
\[
h^*_{\Psi}(x) = \mathrm{sgn}(f^*_{\ell}(x) - \theta^*) = 
\mathrm{sgn}(\eta(x) - \psi^{-1}(\theta^*)),
\]
where the second equality is due to $f^*_{\ell} = \psi(\eta)$
and strict monotonicity of $\psi$. Hence, $h^*_{\Psi}$
is a threshold function on $\eta$.
The proof of Lemma \ref{lemma:main} 
(see Section \ref{sec:proof}) actually specifies the exact value
of the threshold $\theta^*$:
\begin{equation}
\psi^{-1}(\theta^*) = \frac{\Psi(h^*_{\Psi})b_1 - a_1}{\Psi(h^*_{\Psi})(b_1+b_2) - (a_1+a_2)},
\label{eq:optimal_threshold}
\end{equation}
which is in agreement with the result obtained by
 \cite{Natarajan_etal2014}.\footnote{To prove (\ref{eq:optimal_threshold}), \cite{Natarajan_etal2014} require an absolute continuity assumption on the marginal distribution over instances with respect to some dominating measure, so as to guarantee the existence of an appropriate density. Our analysis shows that the existence of a density is not required.}

\begin{table}
\begin{center}
\begin{tabular}{lcc}
\toprule
metric & $\gamma$ & $C$ \\
\midrule
Accuracy & $1$ & $2$ \\[3mm]
$F_{\beta}$-measure & $\beta^2 P$ & $\frac{1+\beta^2}{\beta^2 P}$\\[3mm]
Jaccard similarity & $P$ & $\frac{J^*+1}{P}$ \\[3mm]
AM measure & $2P(1-P)$ & $\frac{1}{2P(1-P)}$ \\[3mm]
Weighted accuracy & $w_1 P + w_2(1-P)$ & $\frac{w_1+w_2}{w_1 P + w_2(1-P)}$\\
%Weighted accuracy & $w_1 P + w_2(1-P)$ & $\frac{\mathrm{WA}^*(w_4+w_3-w_2-w_1)+w_1+w_2}{w_1 P + w_2(1-P)}$\\
\bottomrule
\end{tabular}
\end{center}
\caption{Constants which appear in the bound of Lemma \ref{lemma:main} for several performance metrics.}
\label{tbl:bounds}
\end{table}

To make Lemma \ref{lemma:main} easier to grasp, 
consider a special case when the performance metric $\Psi(\FP,\FN) = 1 - \FP - \FN$ is the classification
accuracy. In this case, (\ref{eq:optimal_threshold}) gives $\Psi^{-1}(\theta^*) = 1/2$.
Hence, we obtained the well-known result that the classifier maximizing
the accuracy is a threshold function on $\eta$ at $1/2$. Then,
Lemma \ref{lemma:main} states that given a real-valued $f$, we should
take a classifier $h_{f,\theta^*}$ which thresholds $f$ at $\theta^* = \psi(1/2)$.
Using Table \ref{tbl:strongly_proper_composite_losses}, one can easily verify that $\theta^* = 0$ for logistic, squared-error and exponential loss.
This agrees with the common approach of thresholding the real-valued classifiers trained by minimizing
these losses at $0$ to obtain the label prediction.
The bounds from the lemma are in this case identical (up to a multiplicative constant)
to the bounds obtained by \citet{Bartlett_etal06}.

Unfortunately, for more complicated performance metrics, the optimal threshold $\theta^*$
is unknown, as (\ref{eq:optimal_threshold}) contains unknown
quantity $\Psi(h^*_{\Psi})$, the value of the metric at optimum. The solution in this case is to, given $f$, 
directly search for a threshold which maximizes $\Psi(h_{f,\theta})$.
This is the main result of the paper:

\begin{theorem}
Given a real-valued function $f$, let $\theta^*_f = \argmax_{\theta} \Psi(h_{f,\theta})$. Then, under the
assumptions and notation from Lemma \ref{lemma:main}:
\[
\regret_{\Psi}(h_{f,\theta^*_f}) \leq C \sqrt{\frac{2}{\lambda}} \sqrt{\regret_{\ell}(f)}.
\]
\label{thm:main} 
\end{theorem}
\begin{proof}
The result follows immediately from Lemma \ref{lemma:main}:
Solving $\max_{\theta}\Psi(h_{f,\theta})$ is equivalent to 
solving $\min_\theta \regret_{\Psi}(h_{f,\theta})$, and 
$\min_\theta \regret_{\Psi}(h_{f,\theta}) \leq \regret_{\Psi}(h_{f,\theta^*})$,
where $\theta^*$ is the threshold given by Lemma \ref{lemma:main}.
\end{proof}
Theorem \ref{thm:main} motivates the following procedure for maximization of $\Psi$:
\begin{enumerate}
\item Find $f$ with small $\ell$-regret, e.g. by 
using a learning algorithm minimizing $\ell$-risk on the training sample.
\item Given $f$, solve $\theta^*_f = \argmax_{\theta} \Psi(h_{f,\theta})$.
\end{enumerate}
Theorem \ref{thm:main} states that the $\Psi$-regret of the classifier obtained by
this procedure
is upperbounded by the $\ell$-regret of the underlying real-valued function. 

We now discuss how to approach step 2 of the procedure in practice.
In principle, this step requires maximizing $\Psi$ defined
through $\FP$ and $\FN$, which are expectations over
an unknown distribution $\Pr(x,y)$. However,
it is sufficient to 
optimize $\theta$ on the empirical counterpart of $\Psi$ 
calculated on a separate validation sample. Let $\mathcal{T} = \{(x_i,y_i)\}_{i=1}^n$ be the validation set
of size $n$.
Define:
\[
\widehat{\FP}(h) = \frac{1}{n} \sum_{i=1}^n \assert{h(x_i) = 1, y_i = -1}, \quad
\widehat{\FN}(h) = \frac{1}{n} \sum_{i=1}^n \assert{h(x_i) = -1, y_i = 1},
\]
the empirical counterparts of $\FP$ and $\FN$, and let $\widehat{\Psi}(h) = \Psi(\widehat{\FP}(h),\widehat{\FN}(h))$
be the empirical counterpart of the performance metric $\Psi$. We now replace step 2 by:
\begin{quote}
 Given $f$ and validation sample $\mathcal{T}$, solve $\widehat{\theta}_f = \argmax_{\theta} \widehat{\Psi}(h_{f,\theta})$.
\end{quote}
In Theorem \ref{thm:test_set_tuning_bound} below, we show that:
\[
\regret_{\Psi}(h_{f,\widehat{\theta}_f}) - \regret_{\Psi}(h_{f,\theta^*_f}) = O\left(\sqrt{\frac{\log n}{n}}\right),
\]
so that tuning the threshold on the validation sample of size $n$ (which results in $\widehat{\theta}_f$) instead of on the population level (which results in $\theta^*_f$)
will cost at most $O\Big(\sqrt{\frac{\log n}{n}}\Big)$ additional regret. The main idea of the proof is that finding the optimal threshold comes down
to optimizing within a class of $\{-1,1\}$-valued threshold functions, which has small Vapnik-Chervonenkis dimension. This, together with the fact
that under assumptions from Lemma \ref{lemma:main}, $\Psi$ is stable with respect to its arguments, implies that $\Psi(h_{f,\widehat{\theta}_f})$ is
close to $\Psi(h_{f,\theta^*_f})$.

\begin{theorem}
\label{thm:test_set_tuning_bound}
Let the assumptions from Lemma \ref{lemma:main} hold, and let:
\[
 D_1 = \sup_{(\FP,\FN)} |b_1 \Psi(\FP,\FN) - a_1|, \quad
 D_2 = \sup_{(\FP,\FN)} |b_2 \Psi(\FP,\FN) - a_2|,
\]
and $D = \max\{D_1,D_2\}$.
Given a real-valued function $f$, and a validation set $\mathcal{T}$ of size $n$ generated i.i.d. from $P(x,y)$,
let $\widehat{\theta}_f = \argmax_{\theta} \widehat{\Psi} (h_{f,\theta})$ be the threshold
maximizing the empirical counterpart of $\Psi$ evaluated on $\mathcal{T}$. Then,
with probability $1-\delta$ (over the random choice of $\mathcal{T}$):
\[
\regret_{\Psi}(h_{f,\widehat{\theta}_f}) \leq C \sqrt{\frac{2}{\lambda}} \sqrt{\regret_{\ell}(f)} + \frac{16D}{\gamma} \sqrt{\frac{4(1+\log n) + 2 \log \frac{16}{\delta}}{n}}.
\]
\end{theorem}
\begin{proof}
For any $\FP$ and $\FN$, we have:
\begin{align*}
\left| \frac{\partial \Psi(\FP,\FN)}{\partial \FP} \right| &= \frac{|a_1 (b_0 + b_1 \FP + b_2 \FN) - b_1 (a_0 + a_1 \FP + a_2 \FN)|}{(b_0 + b_1 \FP + b_2 \FN)^2} \\
&= \frac{|b_1 \Psi(\FP,\FN) - a_1|}{b_0 + b_1 \FP + b_2 \FN} \leq \frac{|b_1 \Psi(\FP,\FN) - a_1|}{\gamma} \leq \frac{D}{\gamma},
\end{align*}
and similarly,
\[
\left|\frac{\partial \Psi(\FP,\FN)}{\partial \FN}\right| = \frac{|b_2 \Psi(\FP,\FN) - a_2|}{b_0 + b_1 \FP + b_2 \FN} \leq \frac{D}{\gamma}.
\]
For any $(\FP,\FN)$ and $(\FP',\FN')$, 
Taylor-expanding $\Psi(\FP,\FN)$ around $(\FP',\FN')$ up to the first order and using the bounds above gives:
\begin{equation}
\Psi(\FP,\FN) \leq \Psi(\FP',\FN') + \frac{D}{\gamma} \left( |\FP - \FP'| + |\FN - \FN'| \right).
\label{eq:psi_Taylor}
\end{equation}
Now, we have:
\begin{align*}
\regret_{\Psi}(h_{f,\widehat{\theta}_f}) 
&= \regret_{\Psi}(h_{f,\theta^*_f}) + \Psi(h_{f,\theta^*_f}) - \Psi(h_{f,\widehat{\theta}_f})   \\
&\leq C \sqrt{\frac{2}{\lambda}} \sqrt{\regret_{\ell}(f)} + \Psi(h_{f,\theta^*_f}) - \Psi(h_{f,\widehat{\theta}_f}),
\end{align*}
where we used Theorem \ref{thm:main}. Thus, it amounts to bound
$\Psi(h_{f,\theta^*_f}) - \Psi(h_{f,\widehat{\theta}_f})$. From the definition
of $\widehat{\theta}_f$, $\widehat{\Psi}(h_{f,\widehat{\theta}_f}) \geq \widehat{\Psi}(h_{f,\theta^*_f})$, hence:
\begin{align*}
\Psi(h_{f,\theta^*_f}) - \Psi(h_{f,\widehat{\theta}_f}) &\leq \Psi(h_{f,\theta^*_f}) - \widehat{\Psi}(h_{f,\theta^*_f}) + \widehat{\Psi}(h_{f,\widehat{\theta}_f}) - \Psi(h_{f,\widehat{\theta}_f})  \\
&\leq 2 \sup_{\theta} \big|\Psi(h_{f,\theta}) - \widehat{\Psi}(h_{f,\theta})\big| \\
&= 2 \sup_{\theta} \big|\Psi(\FP(h_{f,\theta}),\FN(h_{f,\theta})) - \Psi(\widehat{\FP}(h_{f,\theta}),\widehat{\FN}(h_{f,\theta}))\big|,
\end{align*}
where we used the definition of $\widehat{\Psi}$. Using (\ref{eq:psi_Taylor}),
\[
\Psi(h_{f,\theta^*_f}) - \Psi(h_{f,\widehat{\theta}_f}) \leq \frac{2D}{\gamma} \Big(\sup_{\theta}\big|\FP(h_{f,\theta}) - \widehat{\FP}(h_{f,\theta})\big| + \sup_{\theta}\big|\FN(h_{f,\theta}) - \widehat{\FN}(h_{f,\theta})\big| \Big).
\]
Note that the suprema above are on the deviation of empirical mean from the expectation over the class of threshold functions, which has Vapnik-Chervonenkis dimension equal to $2$. Using standard argument
from Vapnik-Chervonenkis theory \citep[see, e.g.,][]{DevroyeGyorfiLugosi96}, with probability $1 - \frac{\delta}{2}$ over the random choice of $\mathcal{T}$:
\[
\sup_{\theta}\big|\FP(h_{f,\theta}) - \widehat{\FP}(h_{f,\theta})\big| \leq 4\sqrt{\frac{4(1+\log n) + 2 \log \frac{16}{\delta}}{n}},
\]
and similarly for the second supremum. Thus, with probability $1-\delta$,
\[
\Psi(h_{f,\theta^*_f}) - \Psi(h_{f,\widehat{\theta}_f}) \leq  \frac{16D}{\gamma} \sqrt{\frac{4(1+\log n) + 2 \log \frac{16}{\delta}}{n}},
\]
which finishes the proof.
\end{proof}
We note that, contrary to a similar results by \citet{Natarajan_etal2014}, Theorem \ref{thm:test_set_tuning_bound} does not require continuity of the cumulative distribution of $\eta(x)$ around $\theta^*$.

\section{Proof of Lemma \ref{lemma:main}}
\label{sec:proof}
The proof can be skipped without affecting the flow of later
sections.
The proof consists of two steps. 
First, we bound the $\Psi$-regret of any classifier $h$ by
its cost-sensitive classification regret (introduced below). 
Next, we show that there exists a threshold $\theta^*$, such that for any $f$, 
the cost-sensitive classification regret of $h_{f,\theta^*}$ 
is upperbounded by the $\ell$-regret of $f$.
These two steps will be formalized as Proposition \ref{prop:bound_psi_cost_sensitive}
and Proposition \ref{prop:bound_cost_sensitive_ell}.

%\paragraph{Bounding $\Psi$-regret by cost-sensitive regret.}
Given a real number $\alpha \in [0,1]$,
define a \emph{cost-sensitive classification loss}
$\ell_\alpha \colon \{-1,1\} \times \{-1,1\} \to \mathbb{R}_+$ as:
\[
\ell_\alpha(y,\widehat{y}) = \alpha \assert{y=-1}\assert{\widehat{y}=1} + (1-\alpha) \assert{y=1} \assert{\widehat{y}=-1}.
\]
The cost-sensitive loss assigns different costs of misclassification for positive and
negative labels. 
Given classifier $h$, the \emph{cost-sensitive risk} of $h$ is:
\begin{align*}
\risk_\alpha(h) &= \mathbb{E}_{(x,y)}[\ell_\alpha(y,h(x))]\\
&= \alpha \FP(h) + (1-\alpha) \FN(h),
\end{align*}
and the \emph{cost-sensitive regret} is:
\[
\regret_{\alpha}(h) = \risk_\alpha(h) - \risk_\alpha(h^*_\alpha),
\]
where $h^*_{\alpha} = \argmin_h \risk_{\alpha}(h)$.
We now show the following two results:
\begin{proposition}
Let $\Psi$ satisfy the assumptions from Lemma \ref{lemma:main}. Define:
\begin{equation}
 \alpha = \frac{\Psi^* b_1 - a_1}{\Psi^* (b_1 + b_2) - (a_1 + a_2)}.
 \label{eq:alpha}
\end{equation}
Then, $\alpha \in [0,1]$ and for any classifier $h$,
\[
 \regret_{\Psi}(h) \leq C \regret_{\alpha}(h),
\]
where $C$ is defined as in the content of Lemma \ref{lemma:main}.
 \label{prop:bound_psi_cost_sensitive}
\end{proposition}
\begin{proof}
The proof generalizes the proof of Proposition 6 from 
\citet{Puthiya_etal2014}, which concerned the special case of $F_{\beta}$-measure.
For the sake of clarity, we use a shorthand notation
$\Psi = \Psi(h)$, $\Psi^* = \Psi(h^*_{\Psi})$,
$\FP = \FP(h)$, $\FN = \FN(h)$,
$A = a_0 + a_1 \FP + a_2 \FN$,
$B = b_0 + b_1 \FP + b_2 \FN$ for the numerator and
denominator of $\Psi(h)$, and analogously $\FP^*$,
$\FN^*$, 
$A^*$ and $B^*$ for $\Psi(h^*_{\Psi})$.
In this notation:
\begin{align}
\regret_{\Psi}&(h) = \Psi^* - \Psi = \frac{\Psi^* B - A}{B} \nonumber\\
&= \frac{\Psi^* B - A - \overbrace{(\Psi^* B^* - A^*)}^{=0}}{B} \nonumber\\
&= \frac{\Psi^* (B-B^*) - (A - A^*)}{B}\nonumber \\
&= \frac{\left(\Psi^* b_1 - a_1\right) (\FP - \FP^*) + \left(\Psi^* b_2 - a_2\right)(\FN - \FN^*)}{B} \nonumber\\
&\leq \frac{\left(\Psi^* b_1 - a_1\right) (\FP - \FP^*) + \left(\Psi^* b_2 - a_2\right)(\FN - \FN^*)}{\gamma},
\label{eq:derivation_regret_psi}
\end{align}
where the last inequality follows from $B \geq \gamma$ (assumption)
and the fact that $\regret_{\Psi}(h) \geq 0$ for any $h$.
Since $\Psi$ is non-increasing in $\FP$ and $\FN$, we have
\[
\frac{\partial \Psi^*}{\partial \FP^*} = \frac{a_1 B^* - b_1 A^*}{(B^*)^2}
= \frac{a_1 - b_1 \Psi^*}{B^*} \leq 0,
\]
and similarly $\frac{\partial \Psi^*}{\partial \FN^*} = \frac{a_2 - b_2 \Psi^*}{B^*} \leq 0$.
This and the assumption $B^* \geq \gamma$ implies
that both $\Psi^* b_1 - a_1$ and $\Psi^* b_2 - a_2$ are non-negative, so can
be interpreted as misclassification costs. If we normalize the costs by
defining:
\[
 \alpha = \frac{\Psi^* b_1 - a_1}{\Psi^* (b_1 + b_2) - (a_1 + a_2)},
\]
then (\ref{eq:derivation_regret_psi}) implies:
\begin{align*}
\regret_{\Psi}(h) &\leq C \left(\risk_{\alpha}(h) - \risk_{\alpha}(h^*_{\Psi}) \right)  \\
&\leq C \left(\risk_{\alpha}(h) - \risk_{\alpha}(h^*_{\alpha}) \right) = C \regret_{\alpha}(h).  \\
\end{align*}
\end{proof}

%\paragraph{Bounding cost-sensitive regret by $\ell$-regret.}
%We will show the following result:
\begin{proposition}
For any real-valued function $f \colon X \to \mathbb{R}$ any $\lambda$-strongly proper   composite loss $\ell$ with link function $\psi$, and any $\alpha \in [0,1]$:
\begin{equation}
 \regret_{\alpha}(h_{f,\theta^*})
 \leq \sqrt{\frac{2}{\lambda}} \sqrt{\regret_{\ell}(f)},
 \label{eqn:bound_cost_sensitive_ell}
\end{equation}
where $\theta^* = \psi(\alpha)$.
\label{prop:bound_cost_sensitive_ell}
\end{proposition}
\begin{proof}
First, we will show that (\ref{eqn:bound_cost_sensitive_ell}) holds
\emph{conditionally} for every $x$. To this end, we fix $x$ and deal with $h(x) \in \{-1,1\}$,
$f(x) \in \mathbb{R}$ and $\eta(x) \in [0,1]$, using a shorthand notation $h,f,\eta$.

Given $\eta \in [0,1]$ and $h \in \{-1,1\}$, define the \emph{conditional cost-sensitive risk}
as:
\[
 \riskc_\alpha(\eta,h) =
 \alpha (1-\eta) \assert{h=1}
 + (1-\alpha) \eta \assert{h=-1}.
\]
Let $h_{\alpha}^* = \argmin_{h} \riskc_{\alpha}(\eta,h)$.
It can be easily verified that:
\begin{equation}
h_{\alpha}^* = \mathrm{sgn}(\eta - \alpha). 
\label{eq:opt_h_alpha}
\end{equation}
Define the \emph{conditional cost-sensitive regret} as 
\[
\regretc_\alpha(\eta,h) = \riskc_\alpha(\eta,h)  
- \riskc_\alpha(\eta,h_{\alpha}^*).
\]
Note that if $h = h_{\alpha}^*$, then
$\regretc_\alpha(\eta,h) = 0$. Otherwise,
$\regretc_\alpha(\eta,h) = |\eta - \alpha|$, 
so that:
\[
\regretc_\alpha(\eta,h) 
=\assert{h \neq h_{\alpha}^*}  |\eta - \alpha|.
\]
Now assume $h = \mathrm{sgn}(\widehat{\eta} - \alpha)$ for some $\widehat{\eta}$, i.e., $h$ is of the same form as $h_{\alpha}^*$ in (\ref{eq:opt_h_alpha}), with $\eta$ replaced by $\widehat{\eta}$. We show that for such $h$,
\begin{equation}
\regretc_\alpha(\eta,h) \leq |\eta - \widehat{\eta}|.
\label{eq:proof_regret_eta_eta_hat}
\end{equation}
This statement trivially holds when $h = h_{\alpha}^*$. If $h \neq h_{\alpha}^*$,
then $\eta$ and $\widehat{\eta}$ are on the opposite sides of $\alpha$ (i.e. either $\eta \geq \alpha$ and $\widehat{\eta} < \alpha$ or $\eta < \alpha$ and $\widehat{\eta} \geq \alpha$), hence $|\eta - \alpha| \leq |\eta - \widehat{\eta}|$, which proves (\ref{eq:proof_regret_eta_eta_hat}).

Now, we set the threshold to $\theta^* = \psi(\alpha)$, so that given 
$f \in \mathbb{R}$, 
\[
h_{f,\theta^*} = \mathrm{sgn}(f - \theta^*) =
\mathrm{sgn}(f - \psi(\alpha)) =
\mathrm{sgn}(\psi^{-1}(f) - \alpha),
\]
due to strict monotonicity of $\psi$.
Using (\ref{eq:proof_regret_eta_eta_hat})
with $h = h_{f,\theta^*}$ and $\widehat{\eta} = \psi^{-1}(f)$ gives:
\begin{align}
\regretc_\alpha(\eta,h_{f,\theta^*}) &\leq |\eta - \psi^{-1}(f)|
= \sqrt{(\eta - \psi^{-1}(f))^2} \nonumber \\
&\leq \sqrt{\frac{2}{\lambda}} \sqrt{\regretc_{\ell}(\eta,f)},
\label{eq:conditional_bound_cost_sensitive_ell}
\end{align}
and the last inequality follows from strong properness (\ref{eq:strong_properness_of_ell}).

To prove the unconditional statement (\ref{eqn:bound_cost_sensitive_ell}), we take expectation with respect to $x$ on both sides of (\ref{eq:conditional_bound_cost_sensitive_ell}):
\begin{align}
\regret_\alpha(h_{f,\theta^*}) 
&= \mathbb{E}_x \left[ \regretc_\alpha(\eta,h_{f,\theta^*}(x)) \right] \nonumber \\
\text{(by (\ref{eq:conditional_bound_cost_sensitive_ell}))}\qquad 
&\leq \sqrt{\frac{2}{\lambda}} \mathbb{E}_x \left[ \sqrt{\regretc_{\ell}(\eta(x),f(x))}\right] \nonumber \\
&\leq \sqrt{\frac{2}{\lambda}} \sqrt{\mathbb{E}_x \left[ \regretc_{\ell}(\eta(x),f(x))\right]} \nonumber \\
&=  \sqrt{\frac{2}{\lambda}} \sqrt{\regret_{\ell}(f)},
\label{eq:bound_on_cost_sensitive_regret_by_surrogate_regret}
\end{align}
where the second inequality is from Jensen's inequality applied to the concave
function $x \mapsto \sqrt{x}$.

We note that derivation of (\ref{eq:proof_regret_eta_eta_hat}) follows the steps of the proof of Lemma 4 in \citet{Menon_etal2013}, while (\ref{eq:conditional_bound_cost_sensitive_ell}) and (\ref{eq:bound_on_cost_sensitive_regret_by_surrogate_regret}) were shown in the proof of Theorem 13 by \citet{Agarwal2014}. Hence, the proof is essentially a combination of existing results, which are rederived here for for the sake of completeness.
\end{proof}

\begin{proof}[Proof of Lemma \ref{lemma:main}]
Lemma \ref{lemma:main} immediately follows from Proposition \ref{prop:bound_psi_cost_sensitive} and Proposition \ref{prop:bound_cost_sensitive_ell}.
\end{proof}
Note that the proof actually specifies the exact value of the universal threshold, 
$\theta^* = \psi(\alpha)$, where $\alpha$ is given by (\ref{eq:alpha}).

The bound in Lemma \ref{lemma:main} is unimprovable in a sense that there exist
$f$, $\Psi$, $\ell$, and distribution $\Pr(x,y)$ for which the
bound is tight. To see this, take, for instance, squared error loss $\ell(y,f) = (y-f)^2$
and classification accuracy metric $\Psi(\FP,\FN) = 1-\FP-\FN$. The constants in Lemma
 \ref{lemma:main} are equal to $\gamma = 1$, $C = 2$, and $\lambda = 8$
 (see Table \ref{tbl:performance_metrics}), while
 the optimal threshold is $\theta^* = 0$. The bound then simplifies to
 \[
  \regret_{0/1}(\sgn(f)) \leq \sqrt{\regret_{\mathrm{sqr}}(f)},
 \]
which is known to be tight \citep{Bartlett_etal06}.
%\section{Generalization to other metrics and losses}
%\label{sec:generalization}

%When it comes to other surrogate losses, it is easy to exclude
%losses which are not proper composite, i.e. in which the minimizer $f^*_{\ell}$
%is not a strictly increasing transformation of $\eta$ (e.g., hinge loss).
%The argument is similar as before: plugging
%$f^*_{\ell}$ into the bound as in Lemma \ref{lemma:main} makes
%the right-hand side zero, while the left-hand side remains non-zero in general.

%%%%%%%%%%%
%%%%%%%%%%%

\section{Multilabel classification}
\label{sec:multilabel}

In multilabel classification \citep{Dembczynski_et_al_2012a,Puthiya_etal2014,Koyejo15},
the goal is, given an input (feature vector) $x \in X$,
to simultaneously predict the subset $L \subseteq \mathcal{L}$ of the set of $m$
labels $\mathcal{L} = \{\sigma_1,\ldots,\sigma_m\}$. The subset $L$ 
is often called the set of relevant (positive) labels, 
while the complement $\mathcal{L} \setminus L$ is considered as irrelevant (negative) for $x$. 
We identify a set $L$ of relevant labels with a vector $\vecy = (y_1, y_2, \ldots, y_m)$, $y_i \in \{-1,1\}$, in which $y_i = 1$ iff $\sigma_i \in L$. 
%The set of possible labelings is denoted $Y = \{-1,1\}^m$.
%
We assume observations $(x,\vecy)$ are generated i.i.d. according to $\Pr(x,\vecy)$ (note that the labels are not assumed to be independent).
A \emph{multilabel classifier}:
\[
\vech(x)= (h_1(x), h_2(x), \ldots , h_m(x)),
\]
is a mapping $\vech \colon X \to \{-1,1\}^m$, which assigns a (predicted) label subset to each instance $x \in X$. 
For any $i=1,\ldots,m$, the function $h_i(x)$ is thus a binary classifier, which can be evaluated
by means of $\TP_i(h_i)$,$\FP_i(h_i)$,$\TN_i(h_i)$ and $\FN_i(h_i)$, which are true/false positives/negatives
defined with respect to label $y_i$, e.g. $\FP_i(h_i) = \Pr(h_i(x) = 1 \land y_i = -1)$.

Let $f_1,\ldots,f_m$ be a set of real-valued functions $f_i \colon X \to \mathbb{R}$, $i=1,\ldots,m$, and 
let $\ell$ be a $\lambda$-strongly proper composite loss for binary classification. For each $i=1,\ldots,m$,
we let $\risk^i_{\ell}(f_i)$ and
$\regret^i_{\ell}(f_i)$ denote the $\ell$-risk and the $\ell$-regret of function $f_i$ with respect to label $y_i$:
\[
 \risk^i_{\ell}(f_i) = \mathbb{E}_{(x,y_i)} \left[\ell(y_i,f_i(x)) \right],
 \qquad
  \regret^i_{\ell}(f_i) =  \risk^i_{\ell}(f_i) - \min_f \risk^i_{\ell}(f).
\]
Note that the problem has been decomposed into $m$ independent binary problems and the functions
can be obtained by training $m$ independent real-valued binary classifiers by minimizing loss $\ell$
on the training sample, one for each out of $m$ labels.

What follows next depends on the way in which the binary classification performance metric
is applied in the multilabel setting. We consider two ways of turning binary classification metric
into multilabel metric:
the macro-averaging and the micro-averaging \citep{retrieval,Puthiya_etal2014,Koyejo15}. 

\subsection{Macro-averaging}
\label{sec:macro-averaging}

Given a binary classification performance metric $\Psi(h) = \Psi(\FP(h),\FN(h)$, and a multilabel
classifier $\vech$, we define the macro-averaged metric $\Psi_{\mathrm{macro}}(\vech)$ as:
\[
 \Psi_{\mathrm{macro}}(\vech) = \frac{1}{m} \sum_{i=1}^m \Psi(h_i) = \frac{1}{m} \sum_{i=1}^m \Psi(\FP_i(h_i),\FN_i(h_i)).
\]
The macro-averaging is thus based on
first computing the performance metric separately for each label, and then
averaging the metrics over the labels. The $\Psi_{\mathrm{macro}}$-regret is then defined as:
\[
 \regret_{\Psi_{\mathrm{macro}}}(\vech) 
 = \Psi_{\mathrm{macro}}(\vech_{\Psi}^*) - \Psi_{\mathrm{macro}}(\vech)
 = \frac{1}{m} \sum_{i=1}^m \left(\Psi(h^*_{\Psi,i}) - \Psi(h_i) \right),
\]
where $\vech_{\Psi}^* = (h^*_{\Psi,1},\ldots,h^*_{\Psi,m})$ is the $\Psi$-optimal multilabel classifier:
\[
 h^*_{\Psi,i} = \argmax_h \Psi(\FP_i(h),\FN_i(h)), \qquad i=1,\ldots,m.
\]
Since the regret decomposes into a weighted sum, it is straightforward
to apply previously derived bound to obtain a regret bound for macro-averaged performance metric.

\begin{theorem}
Let $\Psi(\FP,\FN)$ and $\ell$ satisfy the assumptions of Lemma \ref{lemma:main}.
For a set of $m$ real-valued functions $\{ f_i \colon X \to \mathbb{R} \}_{i=1}^m$,
 let $\theta^*_{f_i} = \argmax_{\theta} \Psi(h_{f_i,\theta})$ for each $i=1,\ldots,m$.
Then the classifier $\vech$ defined as:
\[
 \vech = (h_{f_1,\theta^*_{f_1}}, h_{f_2,\theta^*_{f_2}}, \ldots, h_{f_m,\theta^*_{f_m}}),
\]
achieves the following bound on its $\Psi_{\mathrm{macro}}$-regret:
\[
\regret_{\Psi_{\mathrm{macro}}}(\vech)
 \leq \sqrt{\frac{2}{\lambda}} \frac{1}{m} \sum_{i=1}^m C_i \sqrt{\regret^i_{\ell}(f_i)},
\]
where $C_i = \frac{1}{\gamma}\left(\Psi(h_{\Psi,i}^*)(b_1 + b_2) - (a_1 + a_2)\right)$, $i=1,\ldots,m$.
\label{thm:multilabel_macro}
\end{theorem}
\begin{proof}
The theorem follows from applying Theorem \ref{thm:main} once for each label, and then averaging the bounds over the labels.
\end{proof}

Theorem \ref{thm:multilabel_macro} suggests a straightforward decomposition into $m$ 
independent binary classification
problems, one for each label $y_1,\ldots,y_m$, and running (independently for each problem) the two-step procedure
described in Section \ref{sec:main_result}: For $i=1,\ldots,m$, we learn 
a function $f_i$ with small $\ell$-regret with respect to label $y_i$,
and tune the threshold $\theta^*_{f_i}$ to optimize $\Psi(h_{f_i,\theta})$
(similarly as in the binary classification case, one can show that tuning the threshold
on a separate validation sample is sufficient).
Due to decomposition of $\Psi_{\mathrm{macro}}$
into the sum over the labels, this simple procedure turns out to be sufficient.
As we shall see, the case of micro-averaging becomes more interesting.

\subsection{Micro-averaging}
\label{sec:micro-averaging}

Given a binary classification performance metrics $\Psi(h) = \Psi(\FP(h),\FN(h))$, and a multilabel
classifier $\vech$, we define the micro-averaged metric $\Psi_{\mathrm{micro}}(\vech)$ as:
\[
 \Psi_{\mathrm{micro}}(\vech) = \Psi(\overline{\FP}(\vech),\overline{\FN}(\vech)),
\]
where:
\[
 \overline{\FP}(\vech) = \frac{1}{m} \sum_{i=1}^m \FP_i(h_i), \qquad
 \overline{\FN}(\vech) = \frac{1}{m} \sum_{i=1}^m \FN_i(h_i).
\]
Thus, in the micro-averaging, the false positives
and false negatives are first averaged over the labels, 
and then the performance metric is calculated on these averaged
quantities.
The $\Psi_{\mathrm{micro}}$-regret:
\[
 \regret_{\Psi_{\mathrm{micro}}}(\vech) 
 = \Psi_{\mathrm{micro}}(\vech_{\Psi}^*) - \Psi_{\mathrm{micro}}(\vech),
 \quad
 \text{where} \; \vech_{\Psi}^* = \argmax_{\vech} \Psi_{\mathrm{micro}}(\vech),
\]
does not decompose into the sum over labels anymore. However, we are still able
to obtain a regret bound, reusing the techniques from Section \ref{sec:proof},
and, interestingly, this time 
only a single threshold needs to be tuned and is shared among all labels.
\footnote{The fact that a single threshold is sufficient for consistency of micro-averaged
performance measures was
already noticed by \citet{Koyejo15}.}

\begin{theorem}
Let $\Psi(\FP,\FN)$ and $\ell$ satisfy the assumptions of Lemma \ref{lemma:main}.
For a set of $m$ real-valued functions $\{ f_i \colon X \to \mathbb{R} \}_{i=1}^m$,
 let $\theta^*_f = \argmax_{\theta} \Psi_{\mathrm{micro}}(\vech_{f,\theta})$,
where:
\[
 \vech_{f,\theta} = (h_{f_1,\theta}, h_{f_2,\theta}, \ldots, h_{f_m,\theta}).
\]
Then, the classifier $\vech_{f,\theta^*_f} = (h_{f_1,\theta^*_f}, \ldots, h_{f_m,\theta^*_f})$ achieves the following bound on its $\Psi_{\mathrm{micro}}$-regret:
\[
\regret_{\Psi_{\mathrm{micro}}}(\vech_{f,\theta^*_f})
 \leq \sqrt{\frac{2}{\lambda}} \frac{C}{m} \sum_{i=1}^m  \sqrt{\regret^i_{\ell}(f_i)},
\]
where $C = \frac{1}{\gamma}\left(\Psi_{\mathrm{micro}}(\vech_{\Psi}^*)(b_1 + b_2) - (a_1 + a_2)\right)$.
\label{thm:multilabel_micro}
\end{theorem}
\begin{proof}
The proof follows closely the proof of Lemma \ref{lemma:main}. In fact, only 
Proposition \ref{prop:bound_psi_cost_sensitive} requires modifications,
which are given below.
Take any real values
$\FP,\FN$ and $\FP^*,\FN^*$ (to be specified later) in the domain of $\Psi$, such that:
\begin{equation}
\Psi(\FP^*,\FN^*) - \Psi(\FP,\FN) \geq 0.
\label{eq:positivity_condition}
\end{equation}
Using exactly the same steps as in the derivation (\ref{eq:derivation_regret_psi}),
we obtain:
\[
 \Psi(\FP^*,\FN^*) - \Psi(\FP,\FN) \leq C \left( \alpha (\FP-\FP^*) + (1-\alpha) (\FN - \FN^*) \right),
\]
where: 
\begin{align*}
C &= \frac{1}{\gamma}\left(\Psi(\FP^*,\FN^*) (b_1 + b_2) - (a_1 + a_2)\right),\\
\alpha &= \frac{\Psi(\FP^*,\FN^*)b_1 - a_1}{\Psi(\FP^*,\FN^*)(b_1 + b_2) - (a_1 + a_2)}.
\end{align*}
Now, we take: $\FP^* = \overline{\FP}(\vech_{\Psi}^*), \FN^* = \overline{\FN}(\vech_{\Psi}^*)$,
$\FP = \overline{\FP}(\vech)$ and $\FN = \overline{\FN}(\vech)$ for some  $\vech$. Hence,
(\ref{eq:positivity_condition}) is clearly satisfied as its left-hand side is just
the $\Psi_{\mathrm{micro}}$-regret, $\regret_{\Psi_{\mathrm{micro}}}(\vech)$. 
This means that for any multilabel classifier $\vech$:
\begin{align*}
 \regret_{\Psi_{\mathrm{micro}}}(\vech) &\leq 
 C \left( \alpha (\overline{\FP}(\vech)- \overline{\FP}(\vech_{\Psi}^*)) + (1-\alpha) (\overline{\FN}(\vech) - \overline{\FN}(\vech_{\Psi}^*)) \right) \\
 &= \frac{C}{m} \sum_{i=1}^m \alpha (\FP_i(h_i)- \FP_i(h_{\Psi,i}^*)) + (1-\alpha) (\FN_i(h_i) - \FN_i(h_{\Psi,i}^*)) \\
 &= \frac{C}{m} \sum_{i=1}^m \left(\risk^i_{\alpha}(h_i) - \risk^i_{\alpha}(h_{\Psi,i}^*) \right) \\
 &\leq \frac{C}{m} \sum_{i=1}^m \regret^i_{\alpha}(h_i),
\end{align*}
where $\risk^i_{\alpha}(h_i)$ and $\regret^i_{\alpha}(h_i)$ are the cost-sensitive risk and the cost sensitive regret defined with respect to
label $y_i$:
\[
\risk^i_{\alpha}(h_i) = \mathbb{E}_{(x,y_i)}[\ell_\alpha(y_i,h_i(x))], \qquad
\regret^i_{\alpha}(h_i) = \risk^i_\alpha(h_i) - \min_{h} \risk^i_\alpha(h).
\]
If we now take $h_i = h_{f,\theta^*}$, where $\theta^* = \psi(\alpha)$, $\psi$ being the link function of the loss,
Proposition \ref{prop:bound_cost_sensitive_ell} (applied for each $i=1,\ldots,m$ separately) implies:
\[
\regret^i_{\alpha}(h_{f_i,\theta^*}) \leq \sqrt{\frac{2}{\lambda}} \sqrt{\regret^i_{\ell}(f_i)} .
\]
Together, this gives:
\[
\regret_{\Psi_{\mathrm{micro}}}(\vech_{f,\theta^*})
\leq \sqrt{\frac{2}{\lambda}} \frac{C}{m} \sum_{i=1}^m  \sqrt{\regret^i_{\ell}(f_i)}.
\]
The theorem now follows by noticing that:
\[
\theta^*_f = \argmax_{\theta} \Psi_{\mathrm{micro}}(\vech_{f,\theta}) = \argmin_{\theta} \regret_{\Psi_{\mathrm{micro}}}(\vech_{f,\theta}), 
\]
and thus $\regret_{\Psi_{\mathrm{micro}}}(\vech_{f,\theta^*_f}) \leq \regret_{\Psi_{\mathrm{micro}}}(\vech_{f,\theta^*})$.
\end{proof}

Theorem \ref{thm:multilabel_micro} suggests a decomposition into $m$ 
independent binary classification
problems, one for each label $y_1,\ldots,y_m$, and training $m$ real-valued classifiers
$f_1,\ldots,f_m$ with small $\ell$-regret on the corresponding label. Then, however, contrary to
macro-averaging,
a single threshold, shared among all labels, is tuned by optimizing $\Psi_{\mathrm{micro}}$ on a separate validation
sample.

\section{Empirical results}
\label{sec:experiment}

We perform experiments on synthetic and benchmark data to empirically study
the two-step procedure analyzed in the previous sections. 
To this end, we minimize a surrogate loss in the first step to obtain 
a real-valued function $f$, and 
in the second step, we tune a threshold $\hat{\theta}$
on a separate validation set to optimize a given 
performance metric. 
We use logistic loss in this procedure as a surrogate loss. Recall that logistic loss is $4$-strongly proper composite (see Table~\ref{tbl:strongly_proper_composite_losses}). We compare its performance with \emph{hinge loss}, which is even \emph{not} a proper composite function. 
 %We will show that depending on the problem and function class of $f$ it may fail or still perform sufficiently good. 
%
As our task performance metrics, we take the F-measure ($F_{\beta}$-measure with $\beta=1$)
and the AM measure (which is a special case of Weighted Accuracy with weights $w_1 = P$ and $w_2 = 1-P$).
We could also use the Jaccard similarity coefficient;
it turns out, however, that the threshold optimized for the F-measure 
coincides with the optimal threshold for the Jaccard similarity coefficient
(this is because the Jaccard similarity coefficient is strictly monotonic in the F-measure and vice versa), % this fact can be easily verified based on the results of \cite{Natarajan_etal2014}),
so the latter measure does not give anything substantially different than the F-measure.

The experiments on benchmark data are split into two parts. The first part concerns binary classification problems, while the second part multi-label classification. %In the latter we verify the difference between optimization of micro- and macro-averaged performance measures. 

The purpose of this study is \emph{not} about comparing
the two-step approach with alternative methods; this has already
been done in the previous work on the subject,
see, e.g.,
\citep{Ye2012,Puthiya_etal2014}.
We also note that similar experiments have been performed in the cited papers on the statistical consistency 
of generalized performance metrics  
\citep{Natarajan_etal2014,Narasimhan_etal2014,Puthiya_etal2014, Koyejo15}. 
Therefore, we unavoidably repeat some of the results obtained therein, but the main novelty of the experiments reported here is that we emphasize the difference between strongly proper composite losses and non-proper losses.

\subsection{Synthetic data}

We performed two experiments on synthetic data. The first
experiment deals with a discrete domain in which we learn within
a class of all possible classifiers. The second experiment
concerns continuous domain in which we learn within a restricted class of
linear functions.

\paragraph{First experiment.}
%In the first experiment 
We let the input domain $X$ to be a finite set, consisting of $25$ elements,
 $X=\{1,2,\ldots,25\}$,
and take $\Pr(x)$ to be uniform over $X$, i.e. $\Pr(x=i)=1/25$. 
For each $x \in X$, we randomly draw a value of $\eta(x)$ 
from the uniform distribution on the interval $[0,1]$.
In the first step,
we take an algorithm which minimizes a given surrogate loss $\ell$ within the class
of \emph{all} function $f \colon X \to \mathbb{R}$. 
Hence,
given the training data of size $n$, the algorithm
computes the empirical minimizer of surrogate loss $\ell$ independently
for each $x$.
%In the two-step procedure, we first compute the empirical estimates
%of the optimal real-valued function $f^*_{\ell}(x)$ for a given surrogate loss $\ell$ in every point $x \in X$. 
As surrogate losses, we use logistic and hinge loss. 
In the second step, 
we tune the threshold $\hat\theta$ on a separate validation set, also of size $n$.
For each $n$, we repeat the procedure 100,000 times, averaging over samples and over models
(different random choices of $\eta(x)$).
We start with $n=100$ and increase the number of training examples up to $n=10,000$.
The $\ell$-regret and $\Psi$-regret can be easily computed, as the distribution is known
and $X$ is discrete. 

The results are given in Fig.~\ref{fig:synthetic_data-1}.
The $\ell$-regret goes down to zero for both surrogate losses, which is expected,
since this is the objective function minimized by the algorithm. 
Minimization of logistic loss (left plot) %and tuning the threshold on a separate validation sets
gives vanishing $\Psi$-regret for both the F-measure and the AM measure, as predicted by Theorem \ref{thm:main}.
In contrast, minimization of the hinge loss (right plot) is suboptimal for both task metrics
and gives non-zero $\Psi$-regret even in the limit $n \to \infty$.
This behavior can easily be explained by the fact that hinge loss is
not a proper (composite) loss: the risk minimizer for hinge loss is
given by $f^*_{\ell}(x) = \mathrm{sgn}(\eta(x)-1/2)$ \citep{Bartlett_etal06}.
%not converge to any reversible function of $\eta(x)$, but rather to $\mathrm{sgn}(\eta(x)-1/2)$, taking only two values, $1$ or $-1$.
Hence, the hinge loss minimizer is already a threshold function on $\eta(x)$, with the threshold
value set to $1/2$. If, for a given performance metric $\Psi$,
the optimal threshold $\theta^*$ is
different than $1/2$, the hinge loss minimizer will necessarily have suboptimal $\Psi$-risk.
This is clearly visible for the F-measure.
The better result on the AM measure
is explained by the fact that the average optimal threshold over all models is $0.5$ for this measure,
%(since we use uniform distribution on the interval $[0,1]$ to obtain $\eta(x)$),
so the minimizer of hinge loss is
not that far from the minimizer of AM measure.
\begin{figure}[t]
\begin{center}
\begin{tabular}{cc}
\includegraphics[width = 0.475\textwidth]{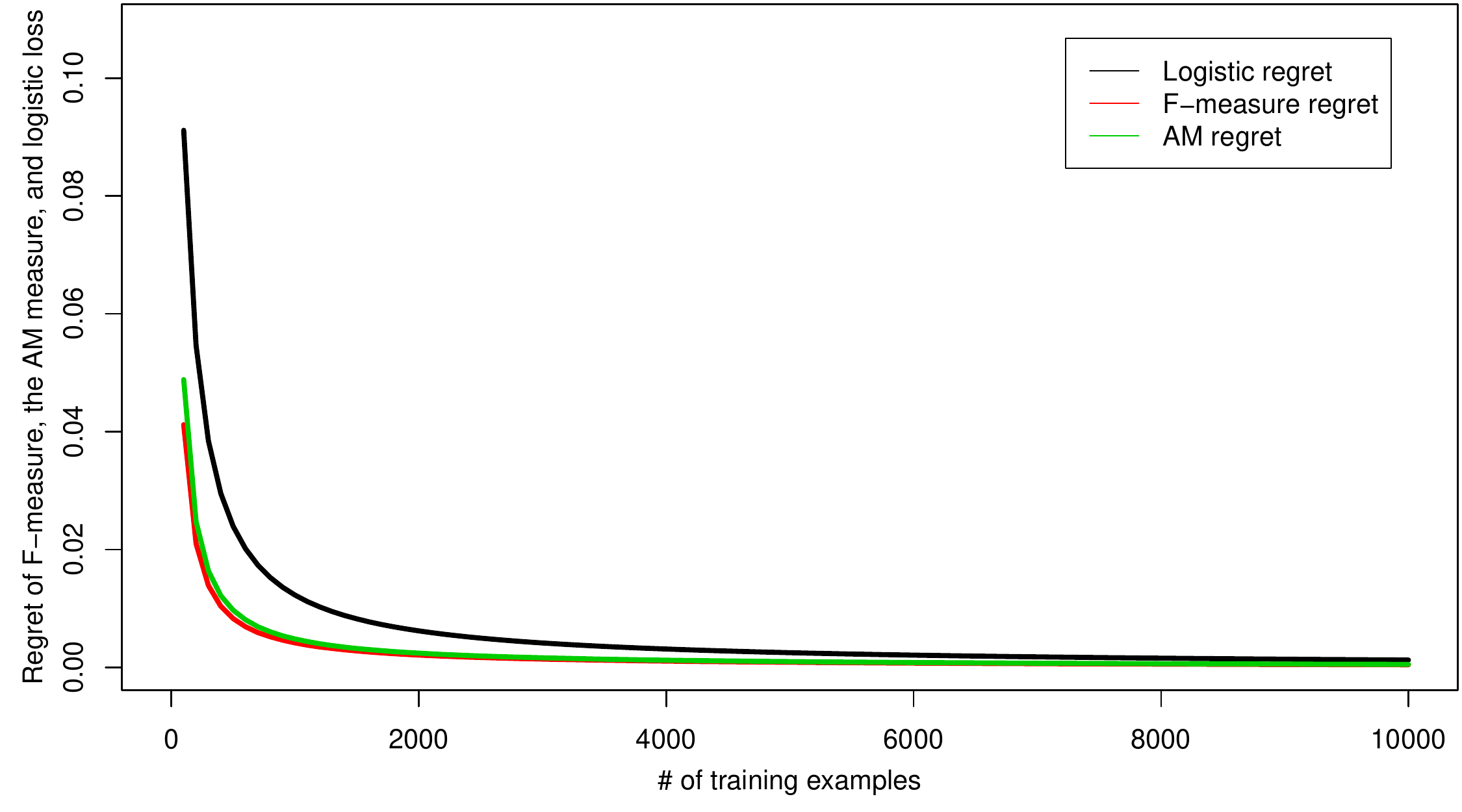} &
\includegraphics[width = 0.475\textwidth]{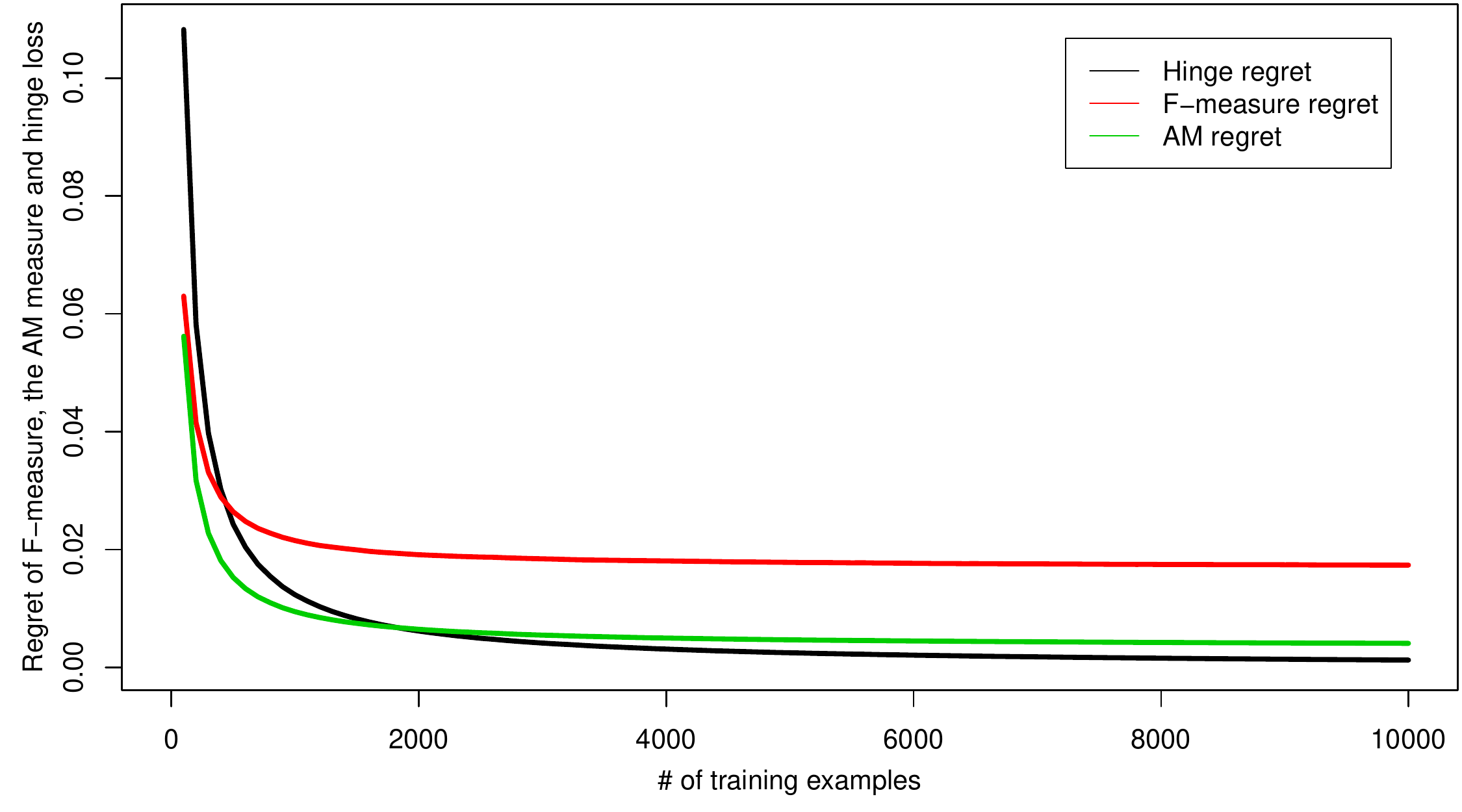} \\
\end{tabular}
\caption{Regret (averaged over 100,000 repetitions)
on the discrete synthetic model as a function
of the number of training examples. 
Left panel: logistic loss is used as a surrogate loss. 
Right panel: hinge loss is used as surrogate loss.}
\label{fig:synthetic_data-1}
\end{center}
%\vspace{-.7cm}
\end{figure}
\begin{figure}[t]
\begin{center}
\begin{tabular}{cc}
\includegraphics[width = .475\textwidth]{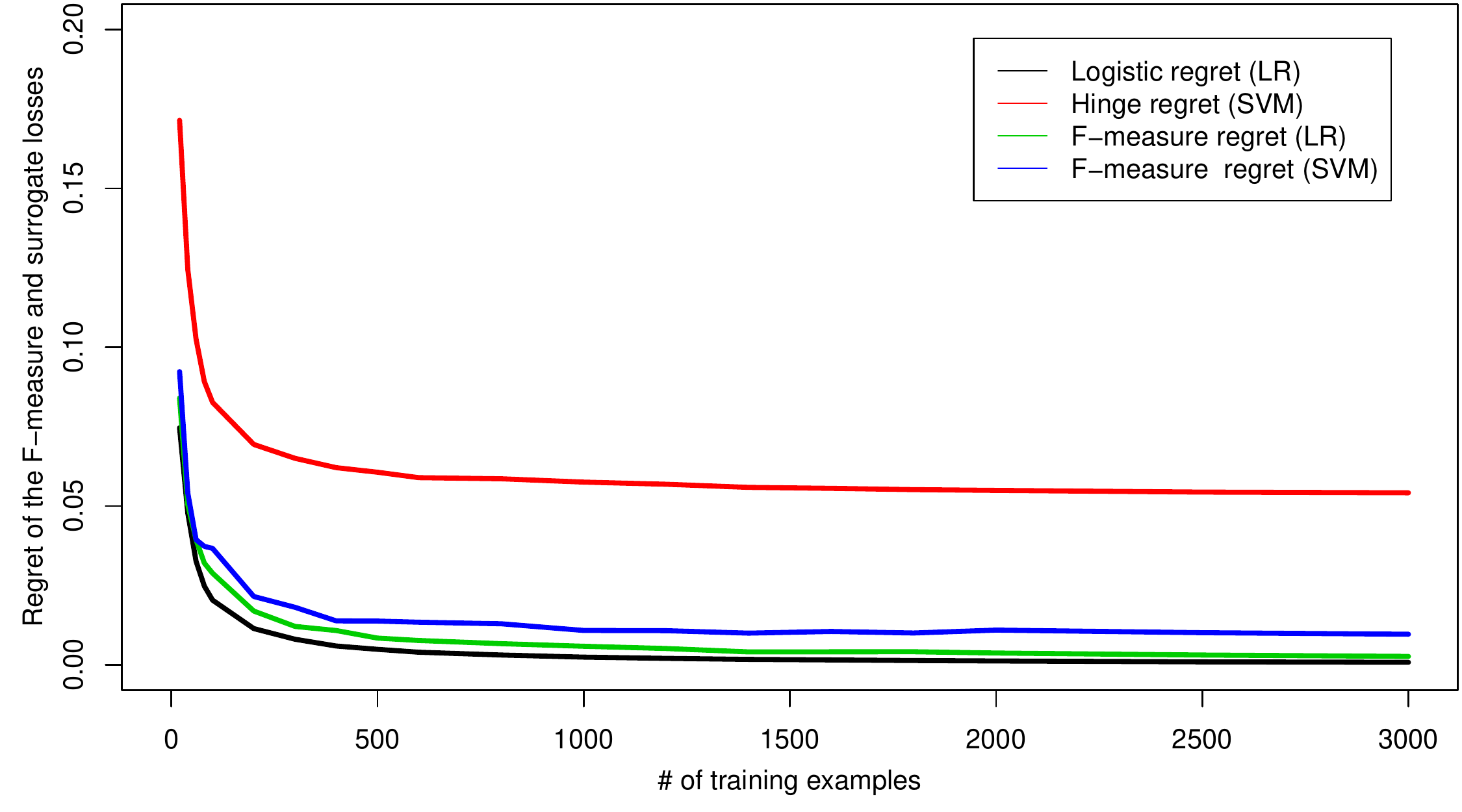} & \includegraphics[width = .475\textwidth]{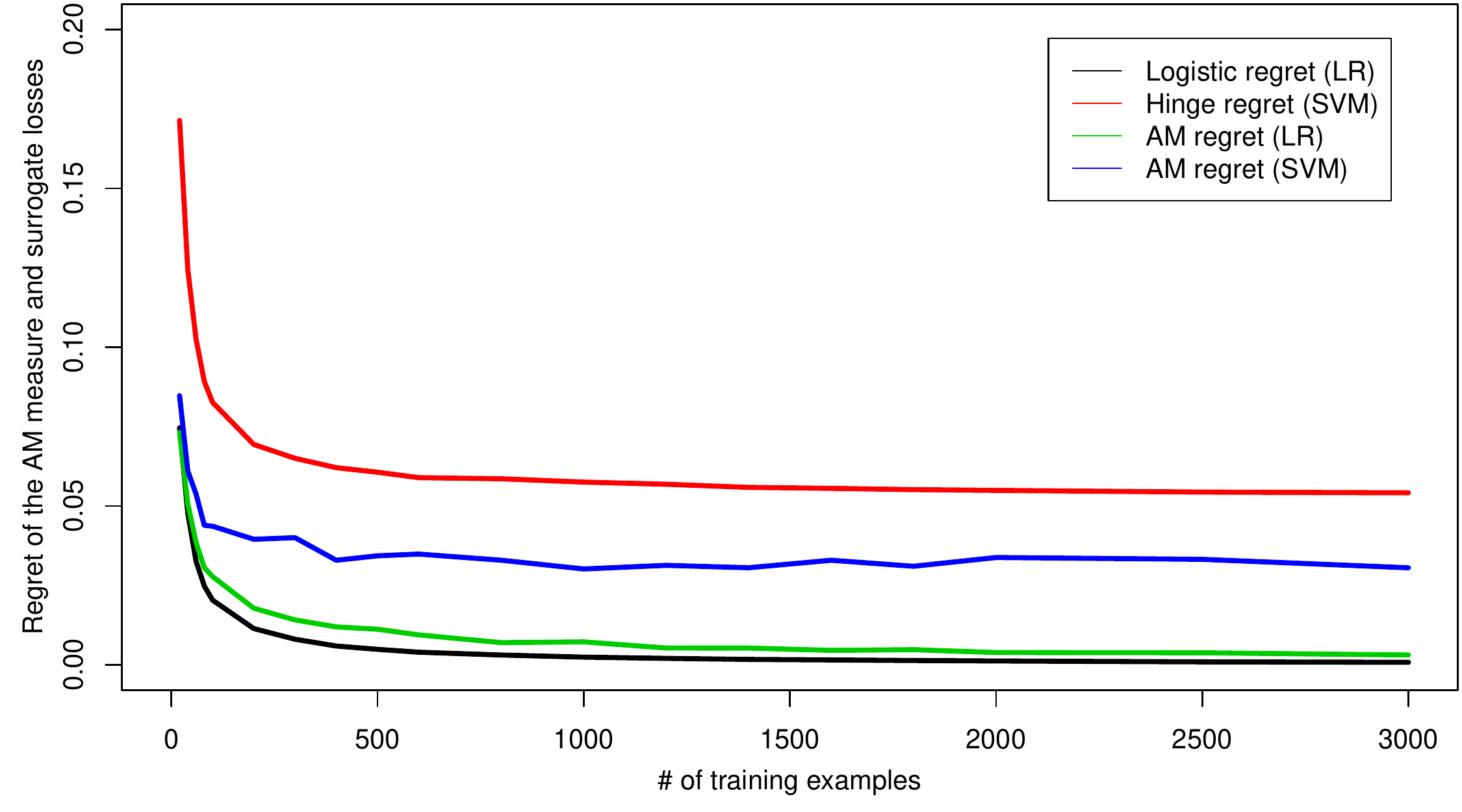} \\
\end{tabular}

\caption{Regret (averaged over 20 x 20 = 400 repetitions)
on the logistic model as a function of the number of training examples. 
Left panel: regret with respect to the F-measure and surrogate losses. 
Right panel: regret with respect to the AM measure and surrogate losses.}
\label{fig:synthetic_data-2}
\end{center}
\end{figure}
\paragraph{Second experiment.}
We take $X = \mathbb{R}^2$ and generate $x \in X$ from a standard Gaussian
distribution. We use a logistic model of the form 
$\eta(x) = \frac{1}{1 + \exp(-a_0 - a^\top x)}$. 
The weights $a = (a_1, a_2)$ and $a_0$ are also drawn from a standard Gaussian.
For a given model (set of weights), we take training sets of increasing size
from $n=100$ up to $n=3000$, using 20 different sets for each $n$.
We also generate one test set of size 100,000.
For each $n$, we use 2/3 of the training data 
to learn a linear model $f(x) = w_0 + w^\top x$,
using either support vector machines (SVM, with linear kernel) or logistic regression (LR). 
We use implementation of these algorithms from the LibLinear package~\citep{liblinear}.\footnote{Software available at \url{http://www.csie.ntu.edu.tw/~cjlin/liblinear}} 
The remaining 1/3 of the training data is used for tuning the threshold. 
We average the results over 20 different models.

The results are given in Fig.~\ref{fig:synthetic_data-2}. 
As before,
we plot the average $\ell$-regret for logistic and hinge loss, and $\Psi$-regret 
for the F-measure and the AM measure. 
The results obtained for LR (logistic loss minimizer) 
agree with our theoretical analysis: 
the $\ell$-regret and $\Psi$-regret with respect to both F-measure and AM measure
go to zero.
This is expected, as the data generating model is a linear logistic model
(so that the risk
minimizer for logistic loss is a linear function),
and thus coincides with a class of functions over which we optimize.
The situation is different for SVM (hinge loss minimizer).
Firstly, the $\ell$-regret for hinge loss does not converge to zero.
This is because
 the risk minimizer for hinge loss is a threshold function $\mathrm{sgn}(\eta(x) - 1/2)$,
and it is not possible to approximate such a function with linear model
$f(x) = w_0 + w^\top x$. Hence, even when $n \to \infty$,
the empirical hinge loss minimizer (SVM) does not converge to the risk minimizer.
This behavior, however, can be \emph{advantageous} for %the performance of 
SVM in terms of the task performance measures. This is because the risk minimizer
for hinge loss, a threshold function on $\eta(x)$ with the threshold value $1/2$,
will perform poorly, for example, in terms of the F-measure and AM measure, for which the optimal threshold $\theta^*$ is usually very different from $1/2$.
In turn, the linear model constraint will prevent convergence to the risk minimizer,
and the resulting linear function $f(x) = w_0 + w^\top x$ will often be close to some reversible
function of $\eta(x)$; hence after tuning the threshold, we will often end up close to the minimizer of a given task performance measure. 
This is seen for the F-measure on the left panel in Fig.~\ref{fig:synthetic_data-2}. In this case, the F-regret of SVM gets quite close to zero, but is still worse than LR. 
The non-vanishing regret is mainly caused by the fact that for some models with imbalanced class priors,
SVM reduce weights $w$ to zero and sets the intercept $w_0$ to $1$ or $-1$,
predicting the same value for all $x \in X$ (this is not caused by a software problem,
it is how the empirical loss minimizer behaves). Interestingly, the F-measure is only slightly  affected by this pathological behavior of empirical hinge loss minimizer. In turn, the AM measure, for which the plots are drawn in the right panel in Fig.~\ref{fig:synthetic_data-2}, is not robust against this behavior of SVM: predicting the majority class actually results in the value of AM measure equal to $1/2$, a very poor performance, which is on the same level as random classifier.

%    is not so vulnerable 

%The performance of SVM in terms of AM measure is much worse and looks quite unstable.
%After inspecting the data, we learned that for some models with imbalanced class priors,
%SVM reduce weights $w$ to zero and sets the intercept $w_0$ to $1$ or $-1$,
%predicting the same value for all $x \in X$ (this is not caused by a software problem,
%it is how the empirical loss minimizer behaves). This phenomenon negatively
%affects the performance in terms of AM measure (and, to a much lesser extent, in terms of F-measure).

\subsection{Benchmark data for binary classification}
\begin{figure*}[t!]
\begin{center}
%\centering
\begin{tabular}{cc}
\texttt{covtype.binary} & \texttt{gisette} \\[3pt]
\includegraphics[width = .45\textwidth]{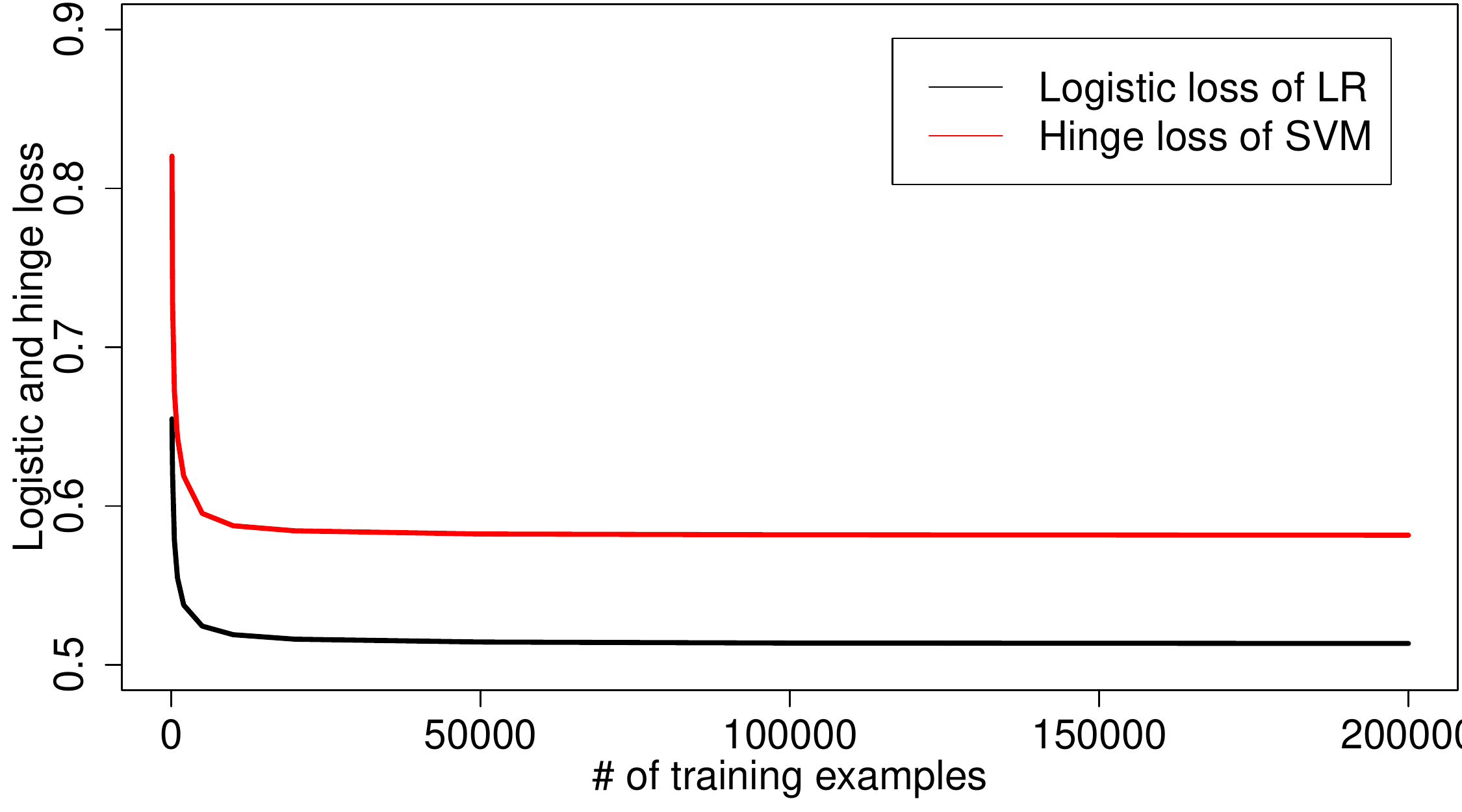} & 
\includegraphics[width = .45\textwidth]{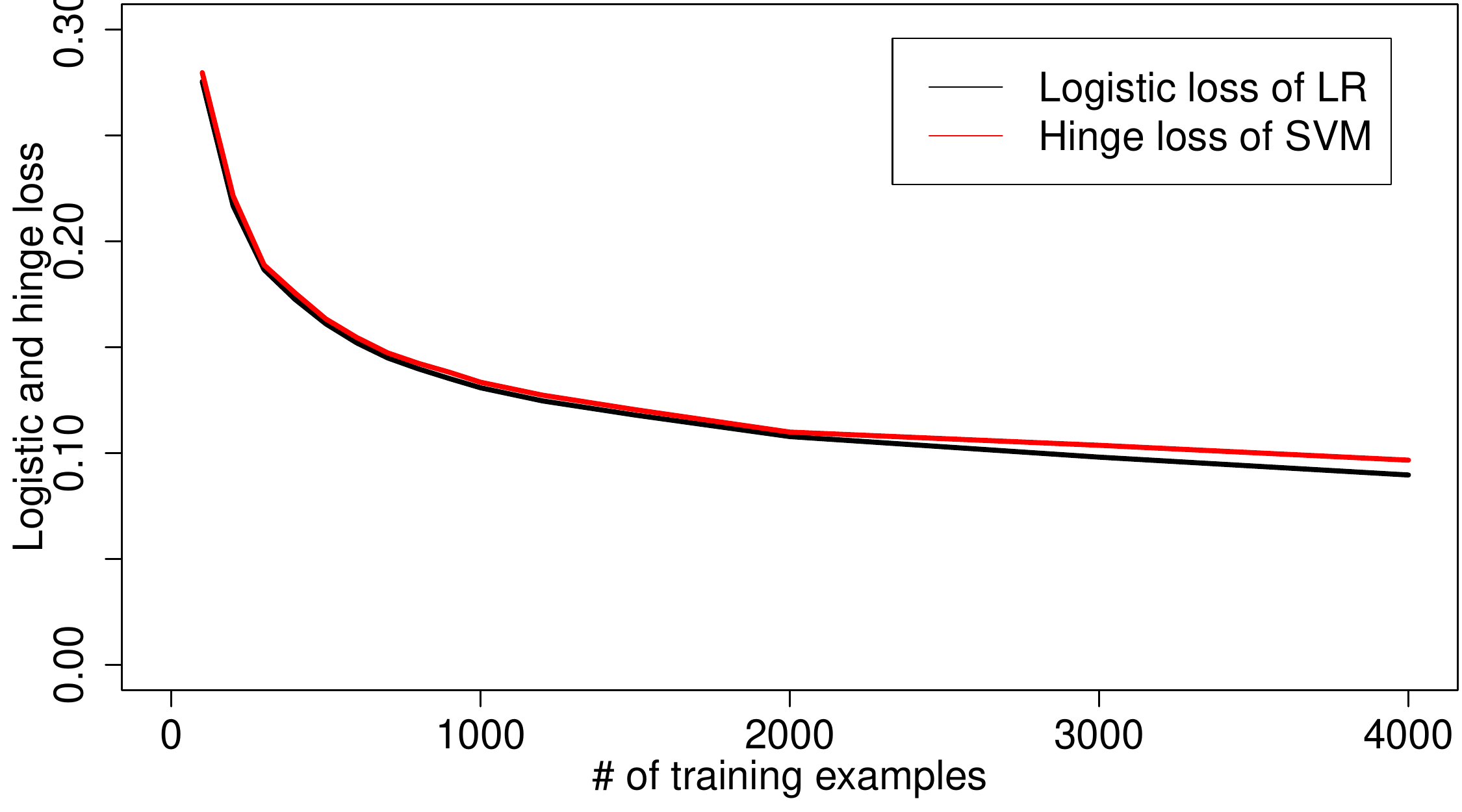} \\
\includegraphics[width = .45\textwidth]{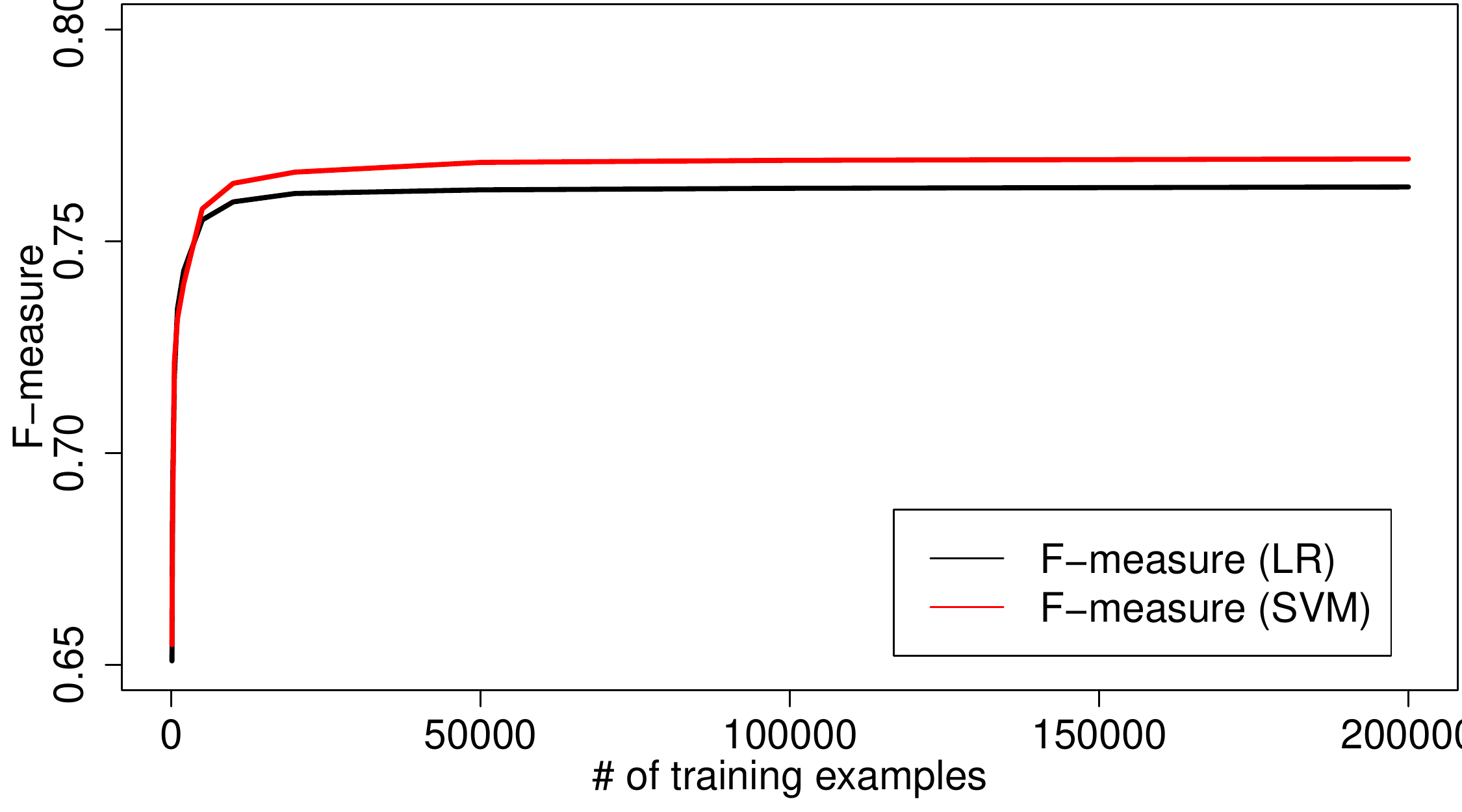} & 
\includegraphics[width = .45\textwidth]{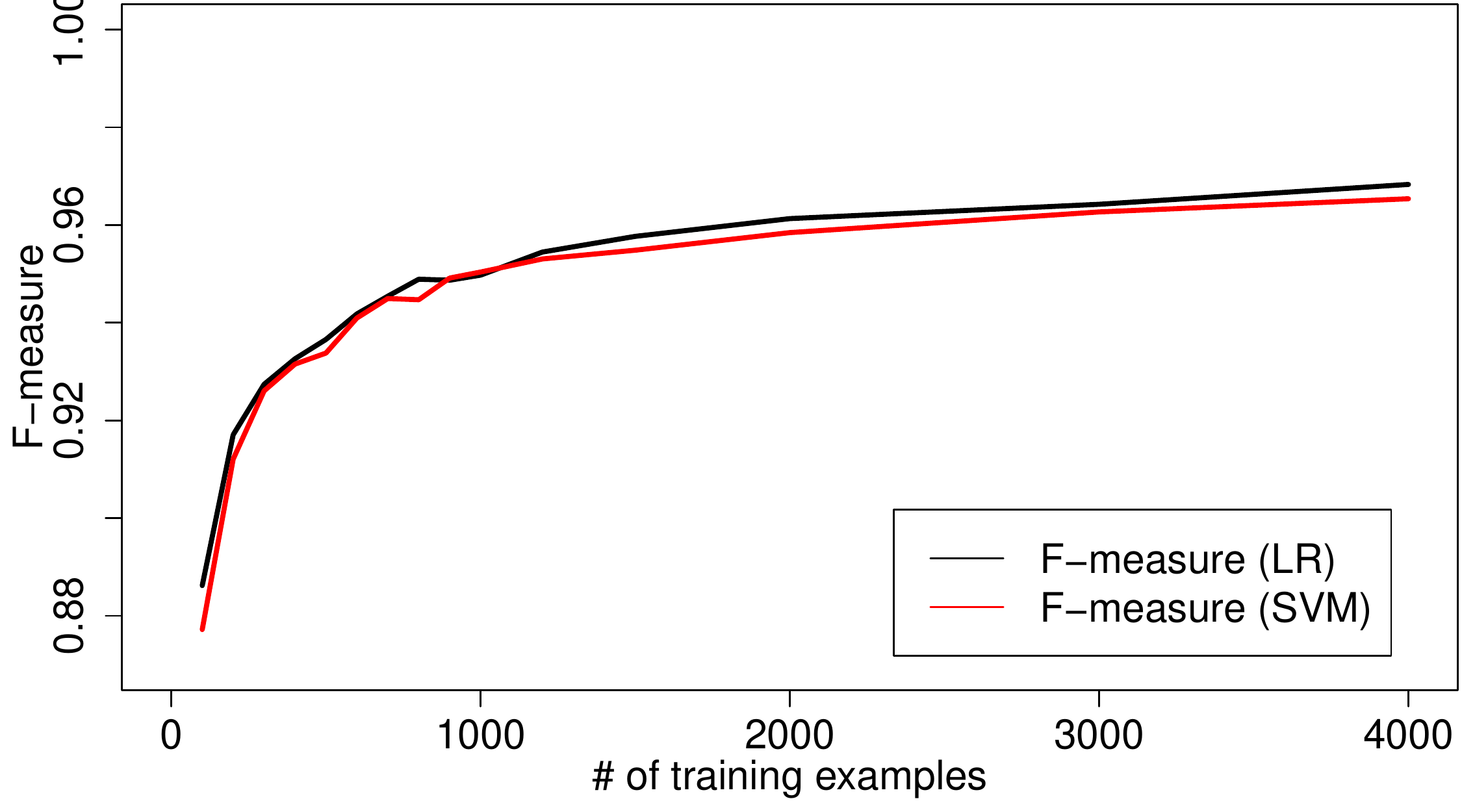} \\\includegraphics[width = .45\textwidth]{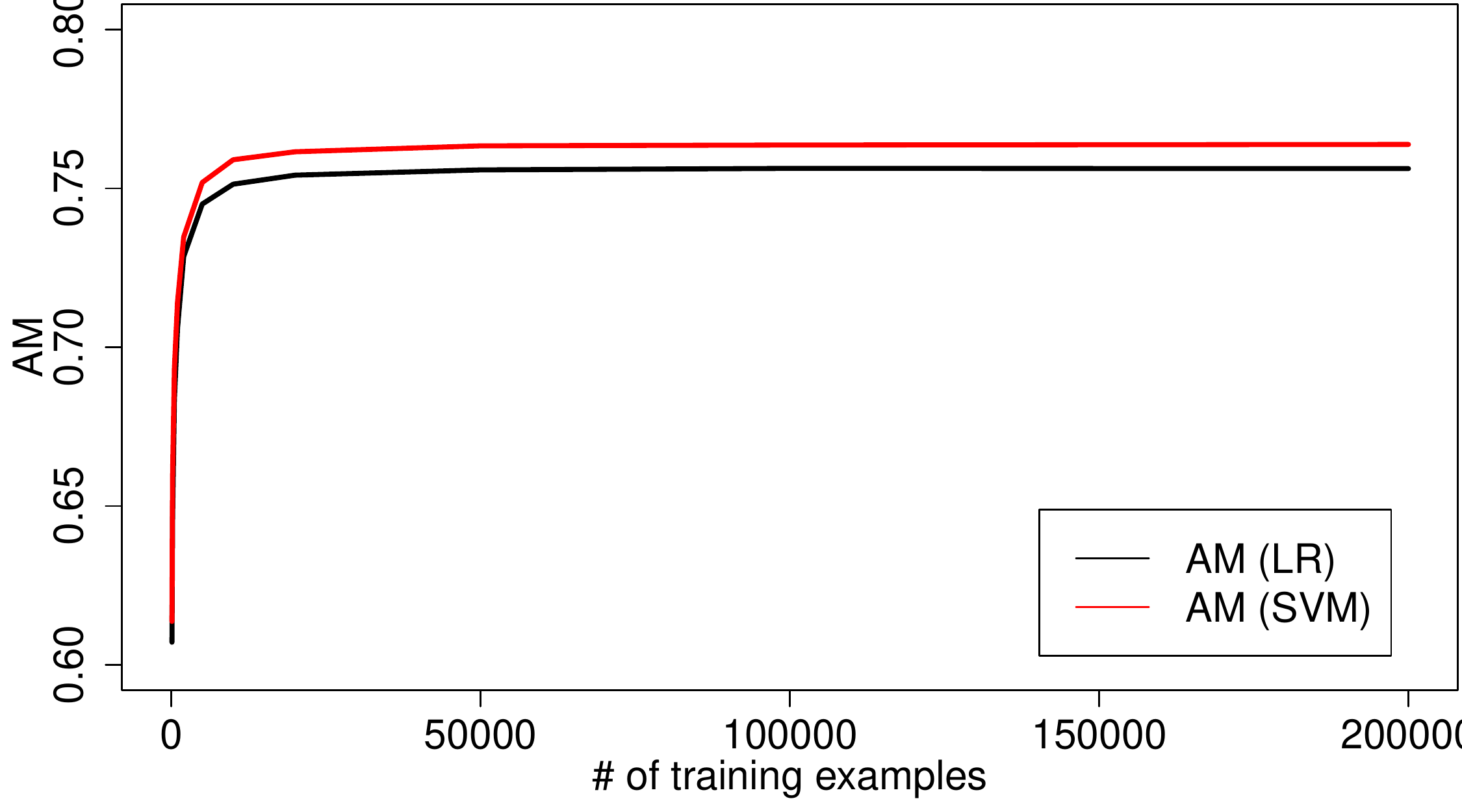} &
\includegraphics[width = .45\textwidth]{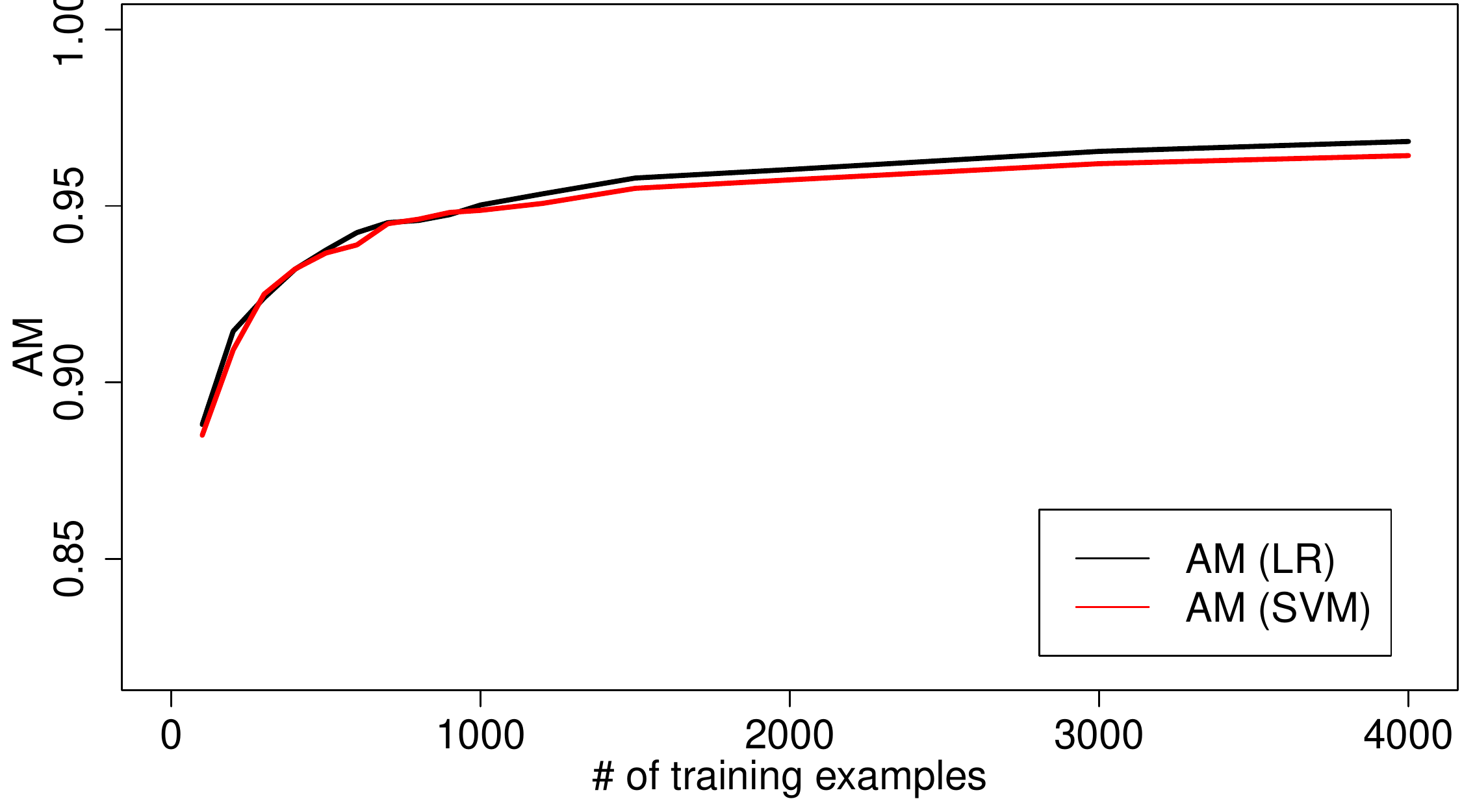} \\
\end{tabular}
\caption{Average test set performance on benchmark data sets as a function of the number of training examples. 
Left panel: \texttt{covtype} dataset. 
Right panel: the \texttt{gisette} dataset. The top plots show logistic and hinge loss, 
the center plots show the F-measure, 
the bottom plots show the AM measure.}
\label{fig:binary_benchmark_results}
\end{center}
\end{figure*}
\begin{table}[t]
\begin{center}
\begin{tabular}{l r r}
\toprule
dataset & \#examples & \#features \\
\midrule
\texttt{covtype}  & 581,012 & 54 \\ 
\texttt{gisette} & 7,000 & 5,000 \\
\bottomrule
\end{tabular}
\end{center}
\caption{Basic statistics for binary classification benchmark datasets}
\label{tbl:binary_data_sets}
%\vspace*{-1cm}
\end{table}
The next experiment is performed on two binary benchmark datasets,\footnote{
Datasets are taken from LibSVM repository: 
\url{http://www.csie.ntu.edu.tw/~cjlin/libsvmtools/datasets}}
described in Table~\ref{tbl:binary_data_sets}. 
We randomly take out a test set of size 181,012 for \texttt{covtype}, and of
size 3,000 for \texttt{gisette}. We use the remaining examples for training.
As before, we incrementally increase the size of the training set.
We use 2/3 of training examples for learning linear model with SVM or LR, and the rest 
for tuning the threshold. We repeat the experiment (random train/validation/test split) 20 times.
The results are plotted 
in Fig~\ref{fig:binary_benchmark_results}. 
Since the data distribution is unknown, we are unable to compute the risk minimizers,
hence we plot the average loss/metric on the test set rather than the regret.
The results show that SVM perform better on the \texttt{covtype} dataset, 
while LR performs better on the \texttt{gisette} dataset. However, 
there is very little difference in performance 
of SVM and LR in terms of the F-measure and the AM measure on these data sets.
We suspect this is due to the fact that $\eta(x)$ function is very different
from linear for these problems, so that neither LR nor SVM converge to the $\ell$-risk
minimizer, and Theorem \ref{thm:main} does not apply. Further studies would be required
to understand the behavior of surrogate losses in this case.
%One can also notice that on the latter problem the regret of surrogate losses goes almost to zero, 
%which is not the case of the former dataset. 
%We present the results in terms of surrogate losses and the F measure and the AM measure computed on the test set. 

\subsection{Benchmark data for multi-label classification}

\begin{figure*}%[t!]
\begin{center}
%\centering
\vspace*{-20pt}
\begin{tabular}{ccc}
\midrule
%\multicolumn{2}{c}{\texttt{scene}} \\[3pt]
\multirow{2}{*}{\begin{sideways}\large{\texttt{scene}}\end{sideways}} &
\includegraphics[width = .37\textwidth]{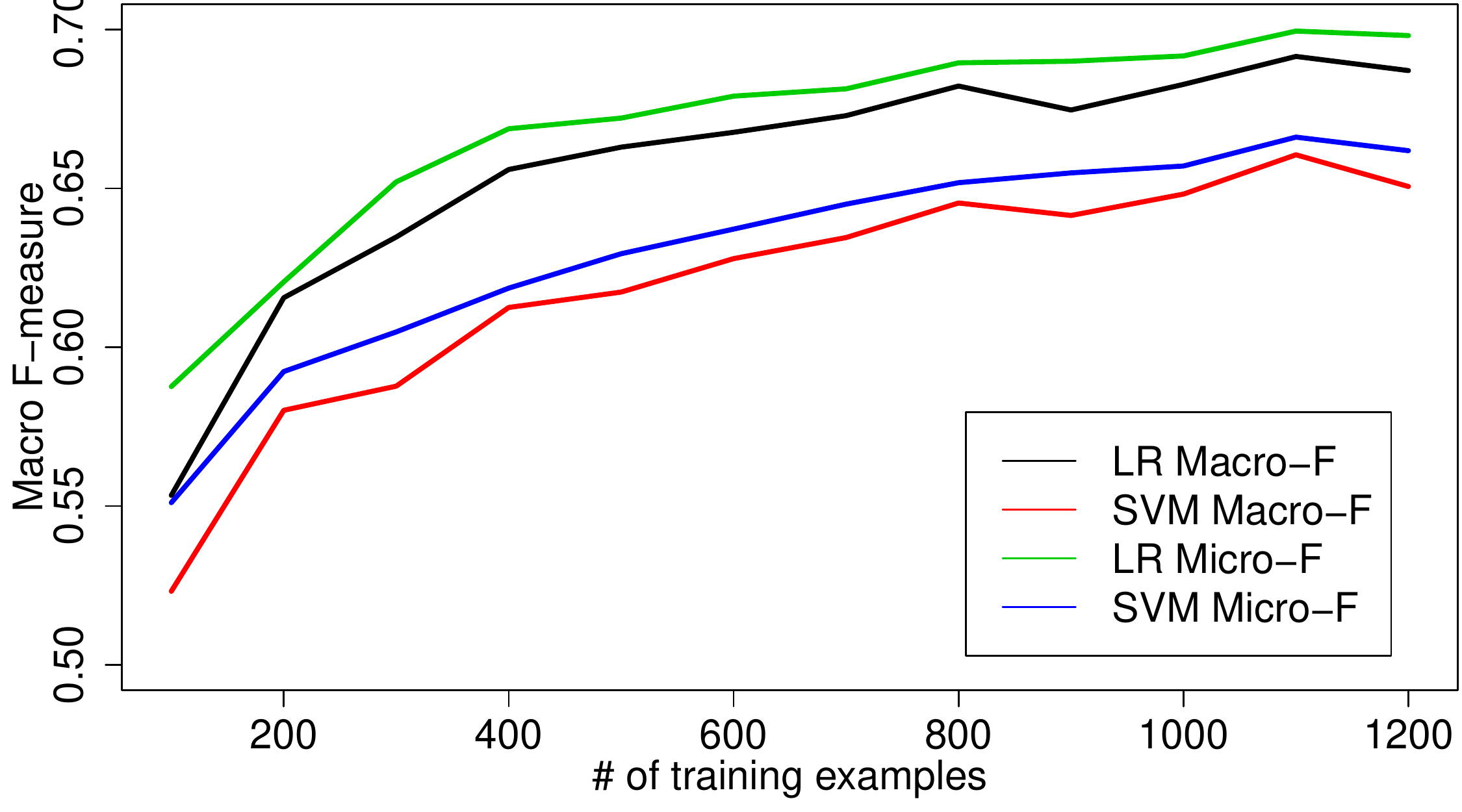} & 
\includegraphics[width = .37\textwidth]{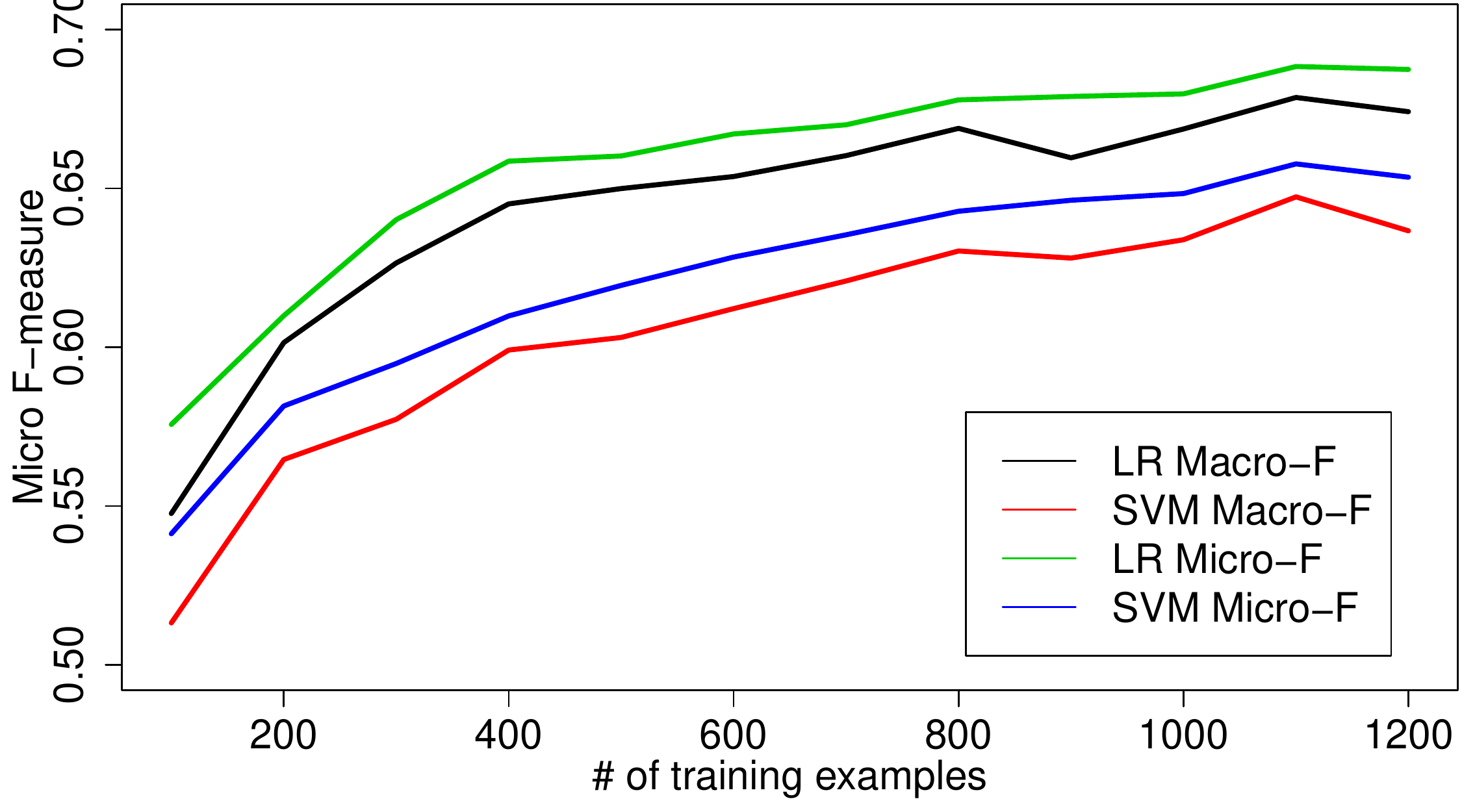} \\
& \includegraphics[width = .37\textwidth]{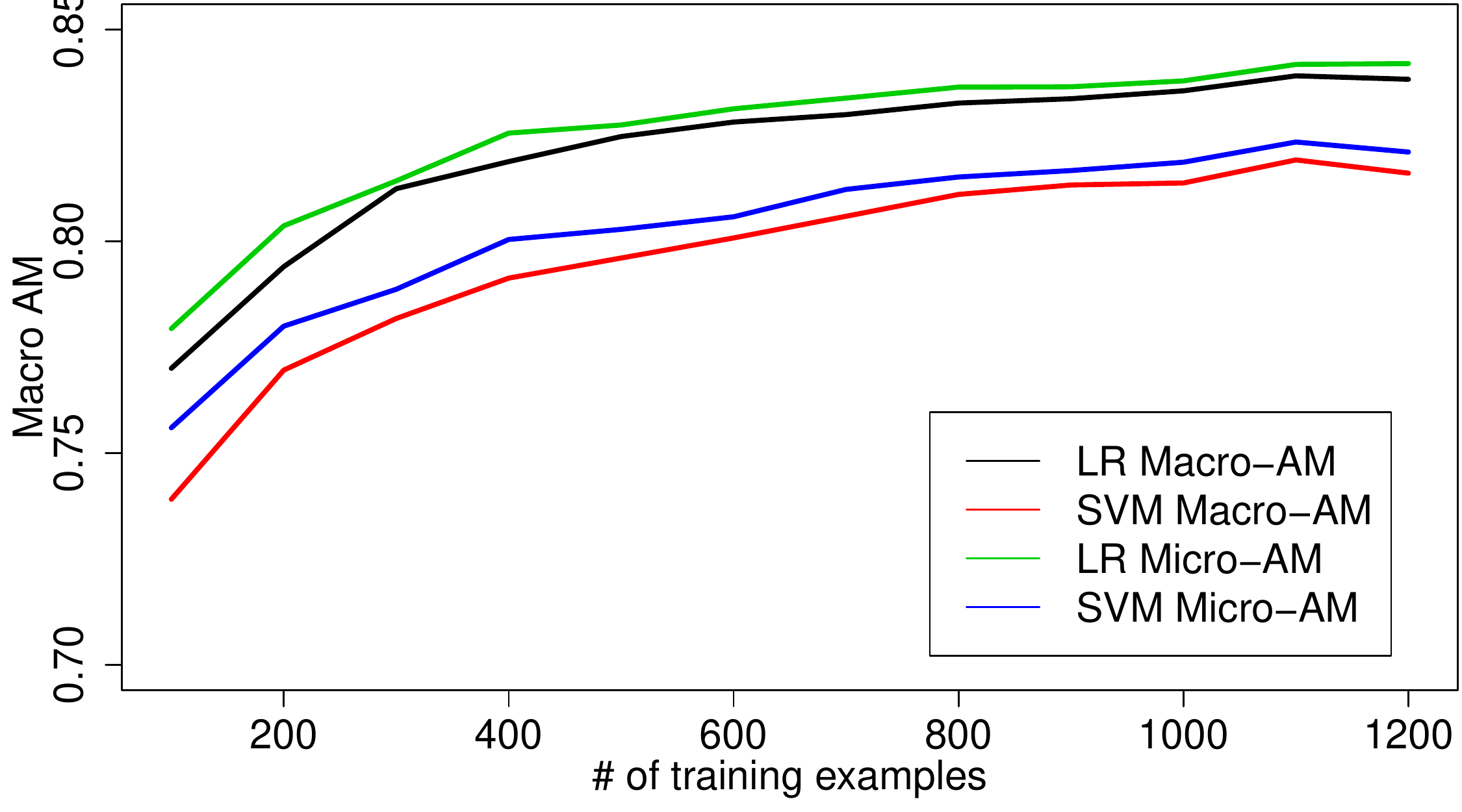} & 
\includegraphics[width = .37\textwidth]{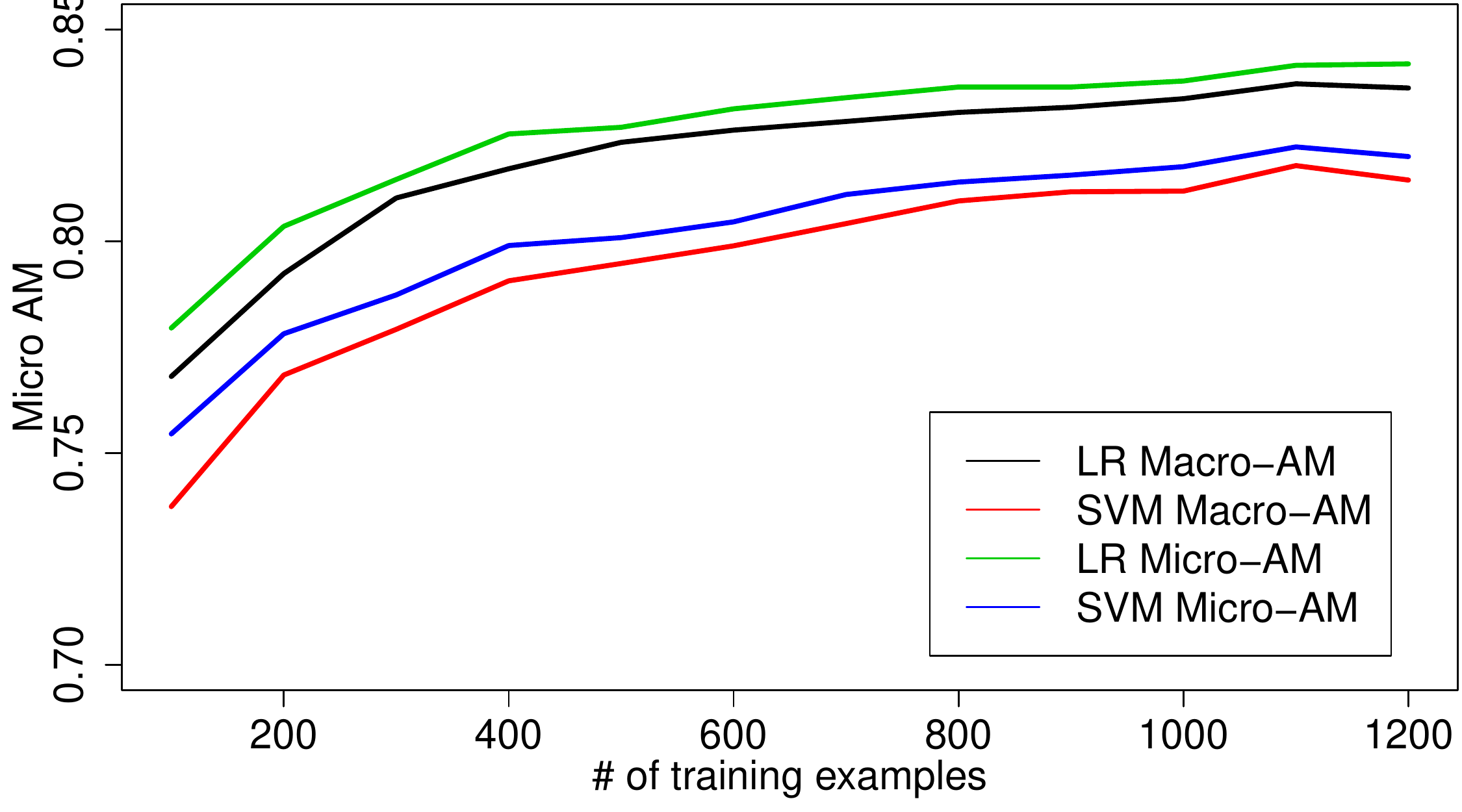} \\
\midrule
%\multicolumn{2}{c}{\texttt{yeast}} \\[3pt]
\multirow{2}{*}{\begin{sideways}\large{\texttt{yeast}}\end{sideways}} &
\includegraphics[width = .37\textwidth]{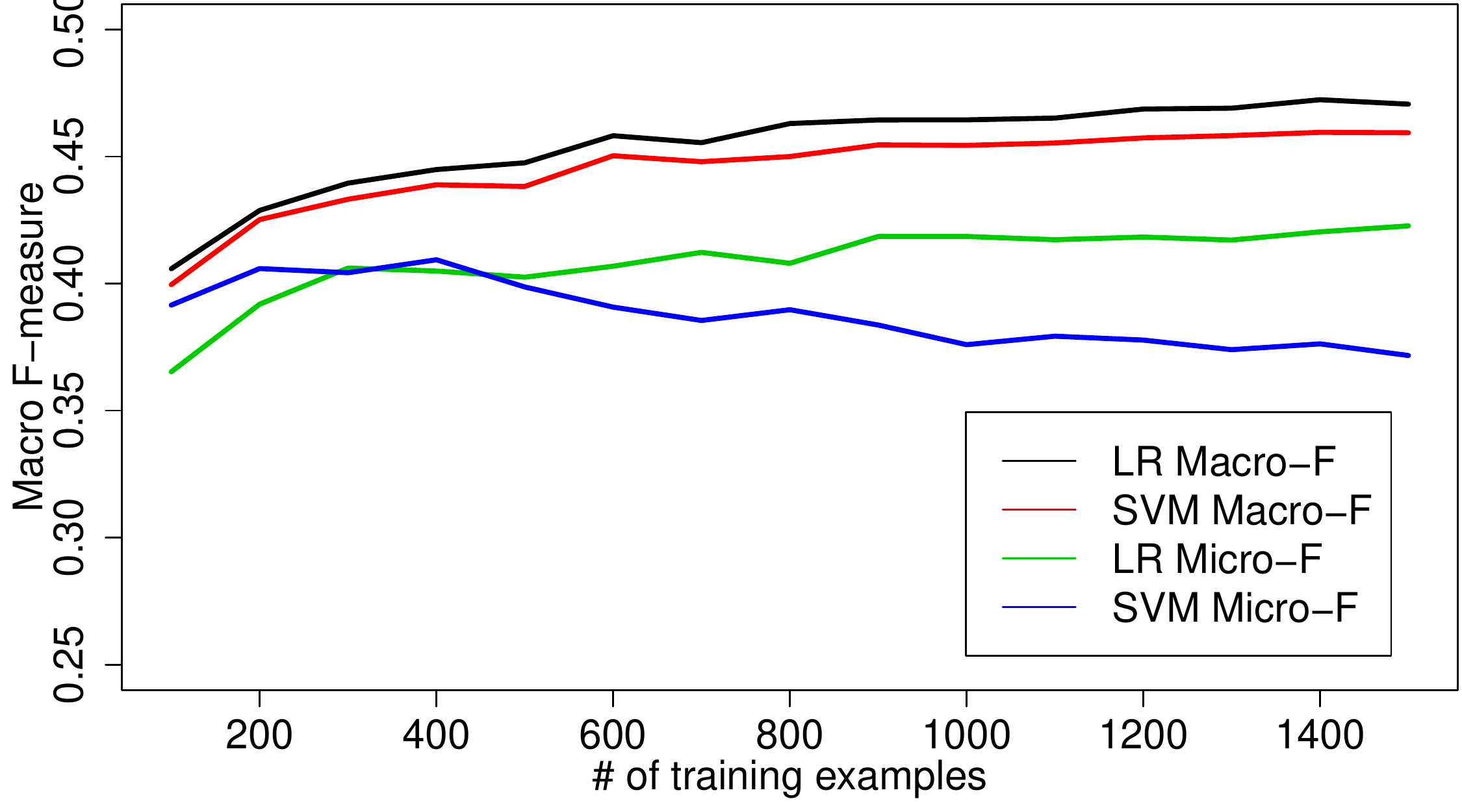} & 
\includegraphics[width = .37\textwidth]{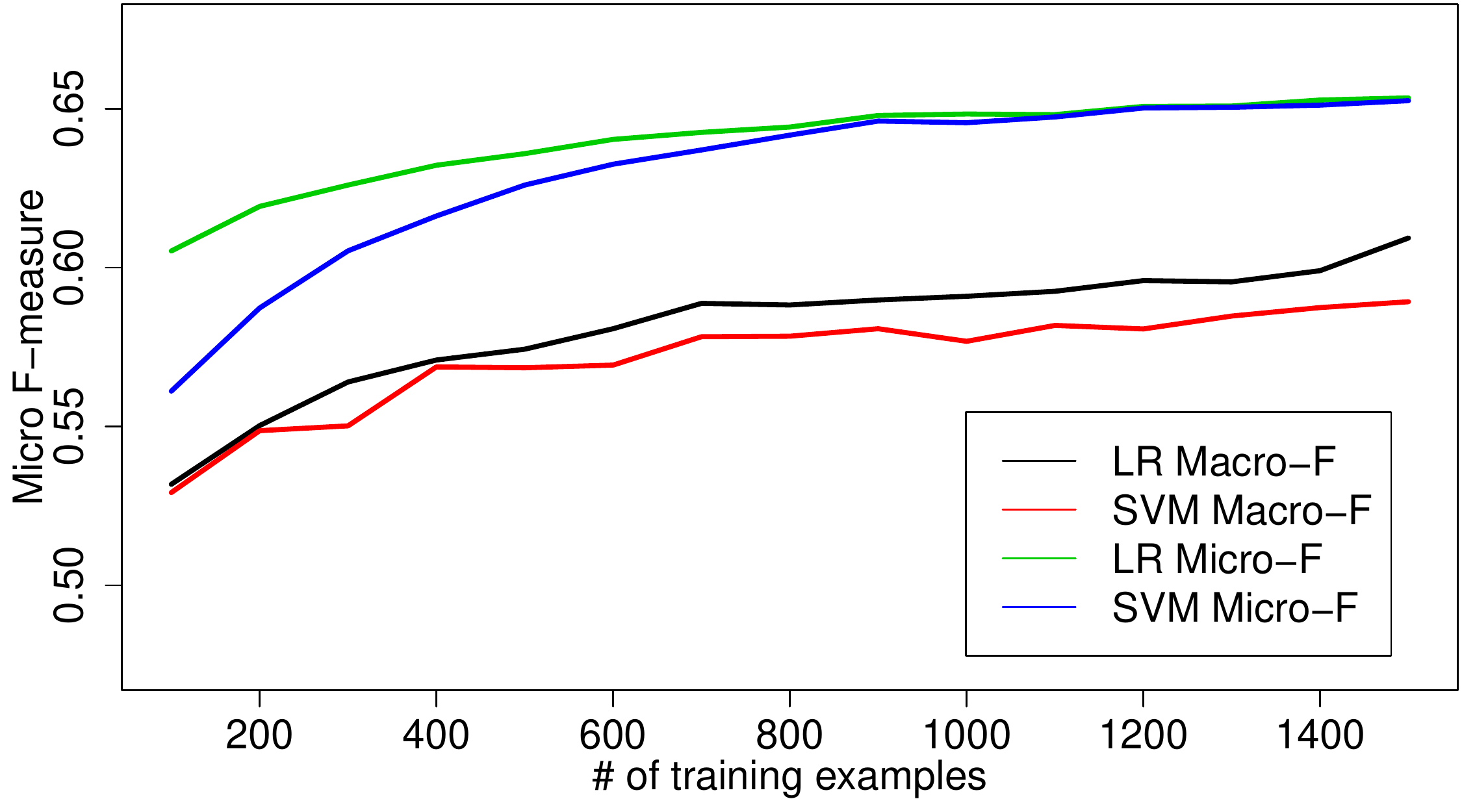} \\
& \includegraphics[width = .37\textwidth]{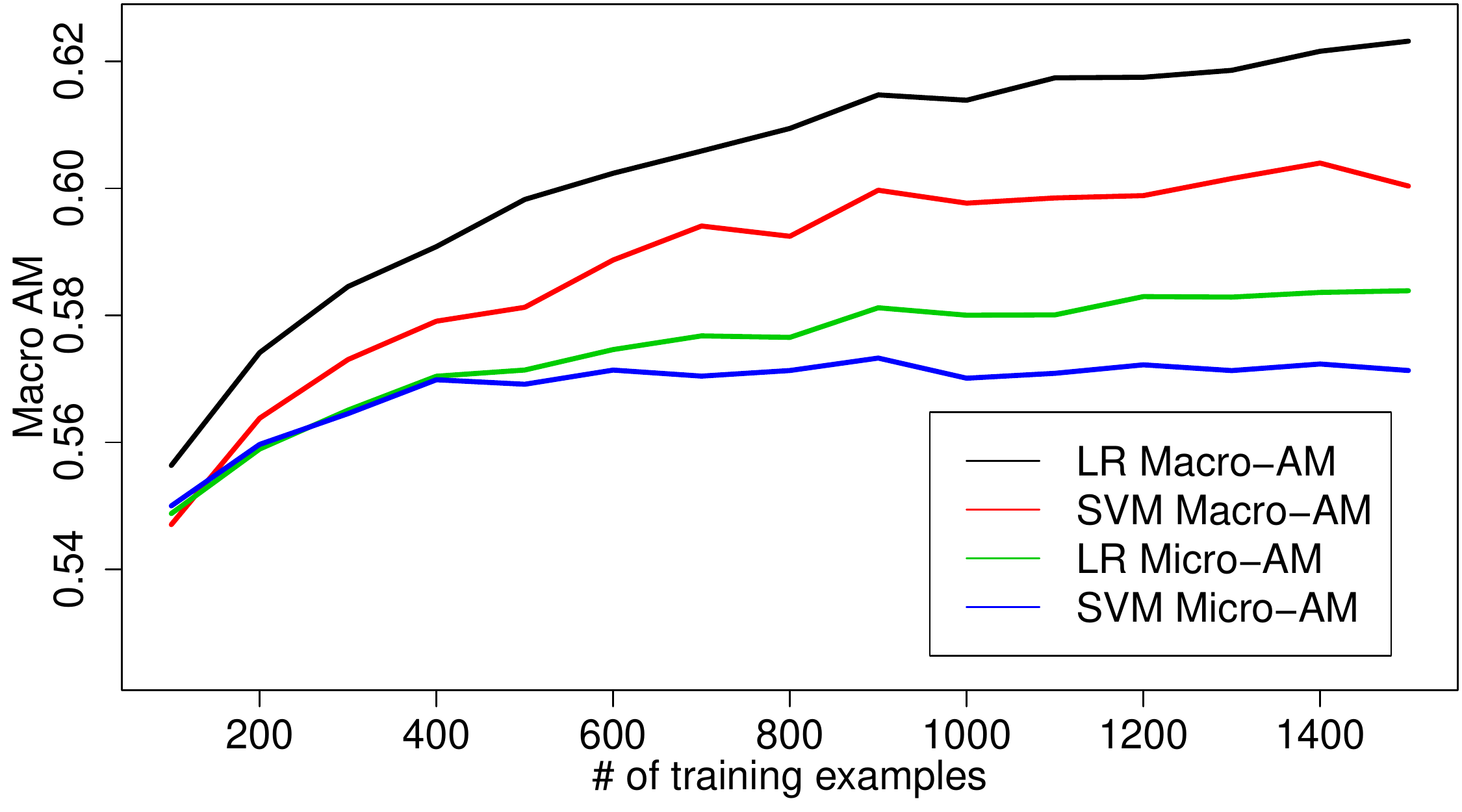} & 
\includegraphics[width = .37\textwidth]{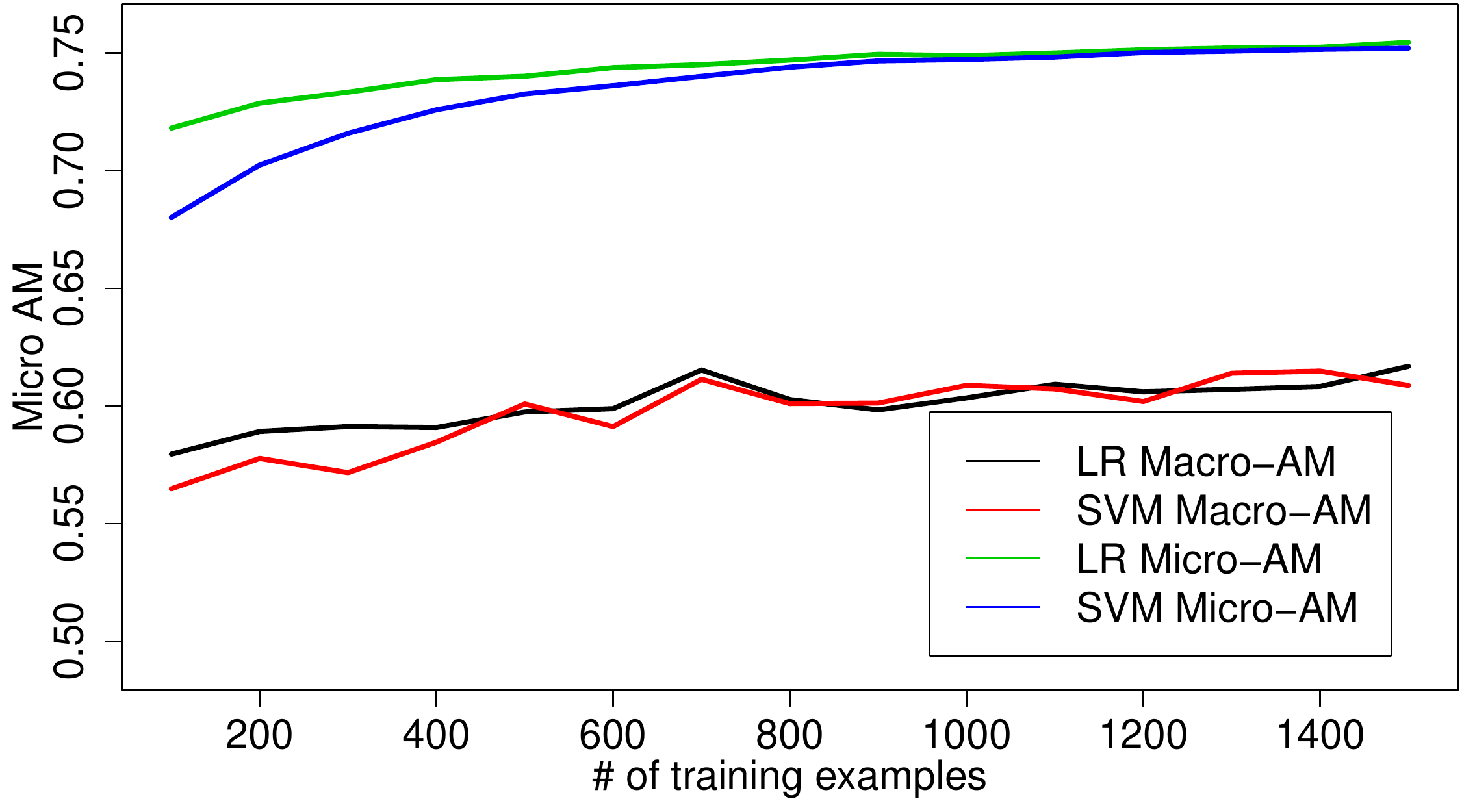} \\
\midrule
%\multicolumn{2}{c}{\texttt{mediamill}} \\[3pt]
\multirow{2}{*}{\begin{sideways}\large{\texttt{mediamill}}\end{sideways}} &
\includegraphics[width = .37\textwidth]{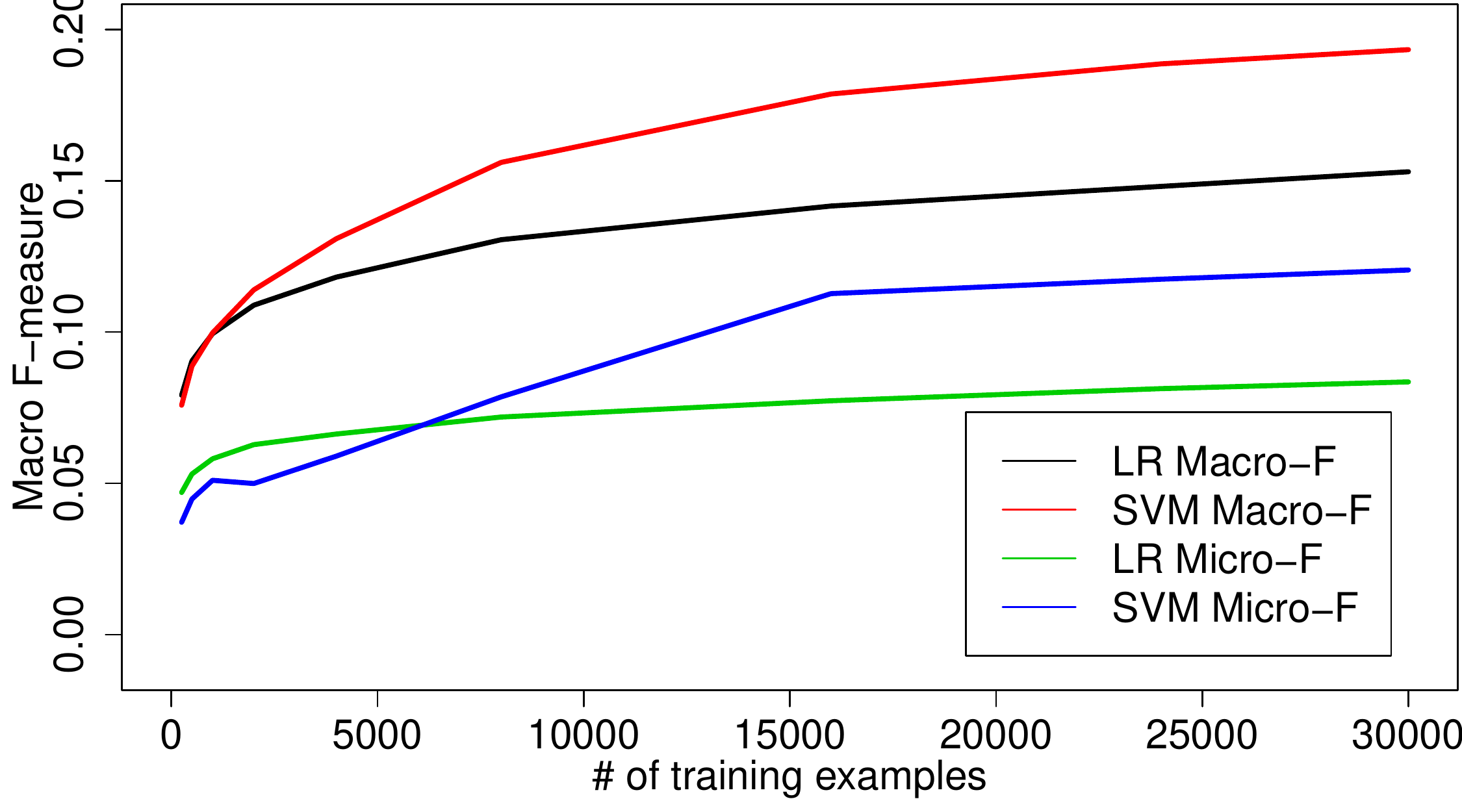} & 
\includegraphics[width = .37\textwidth]{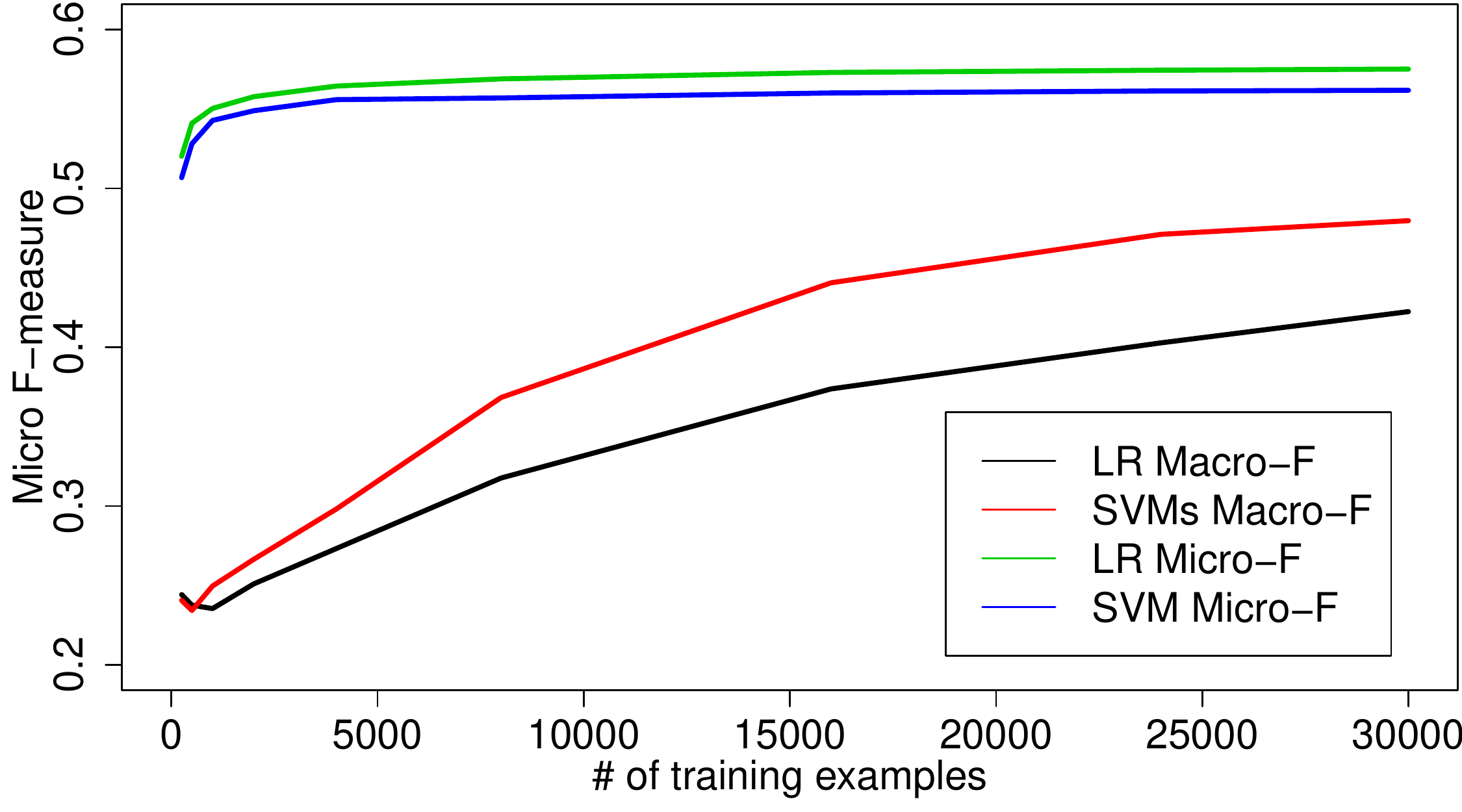} \\
& \includegraphics[width = .37\textwidth]{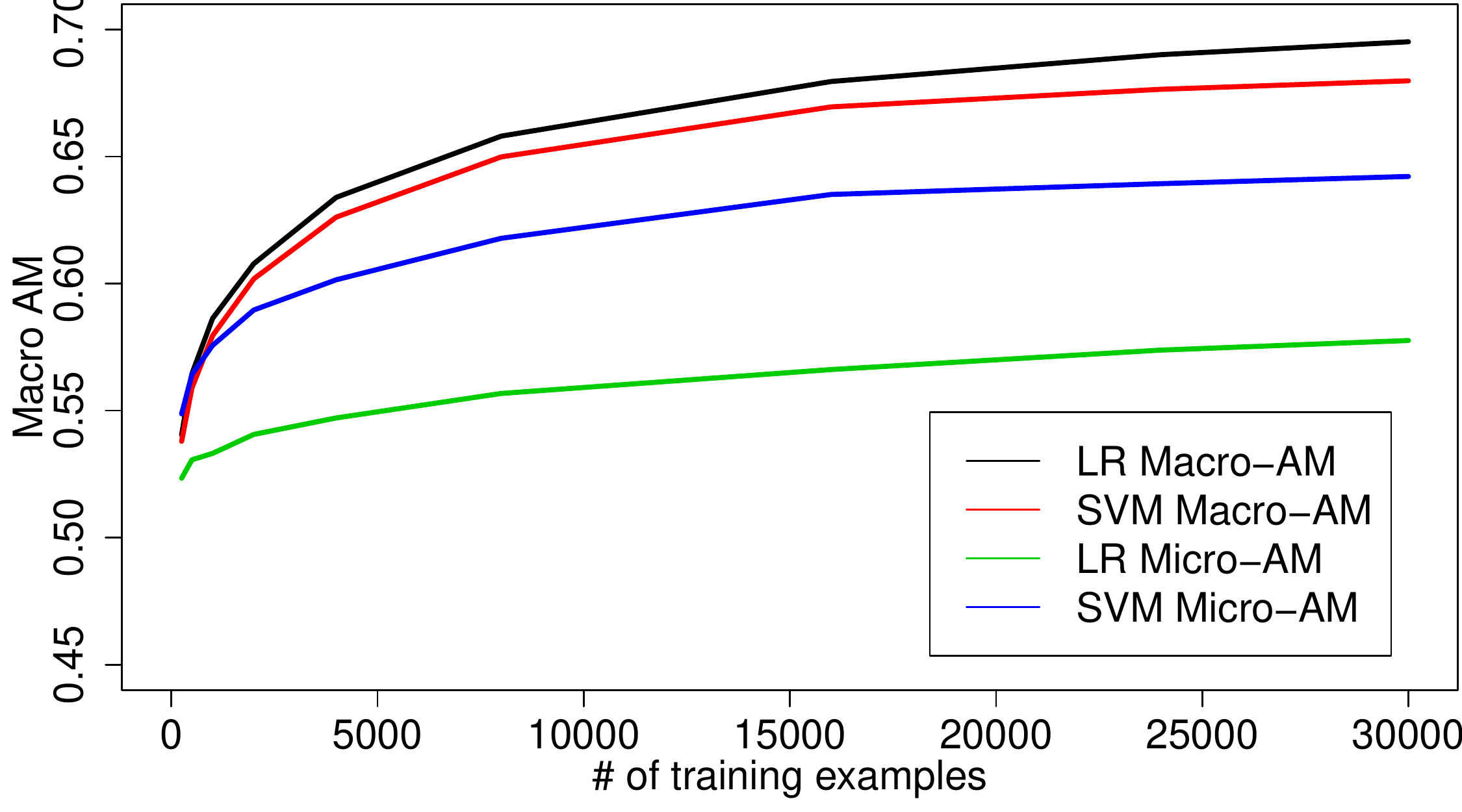} & 
\includegraphics[width = .37\textwidth]{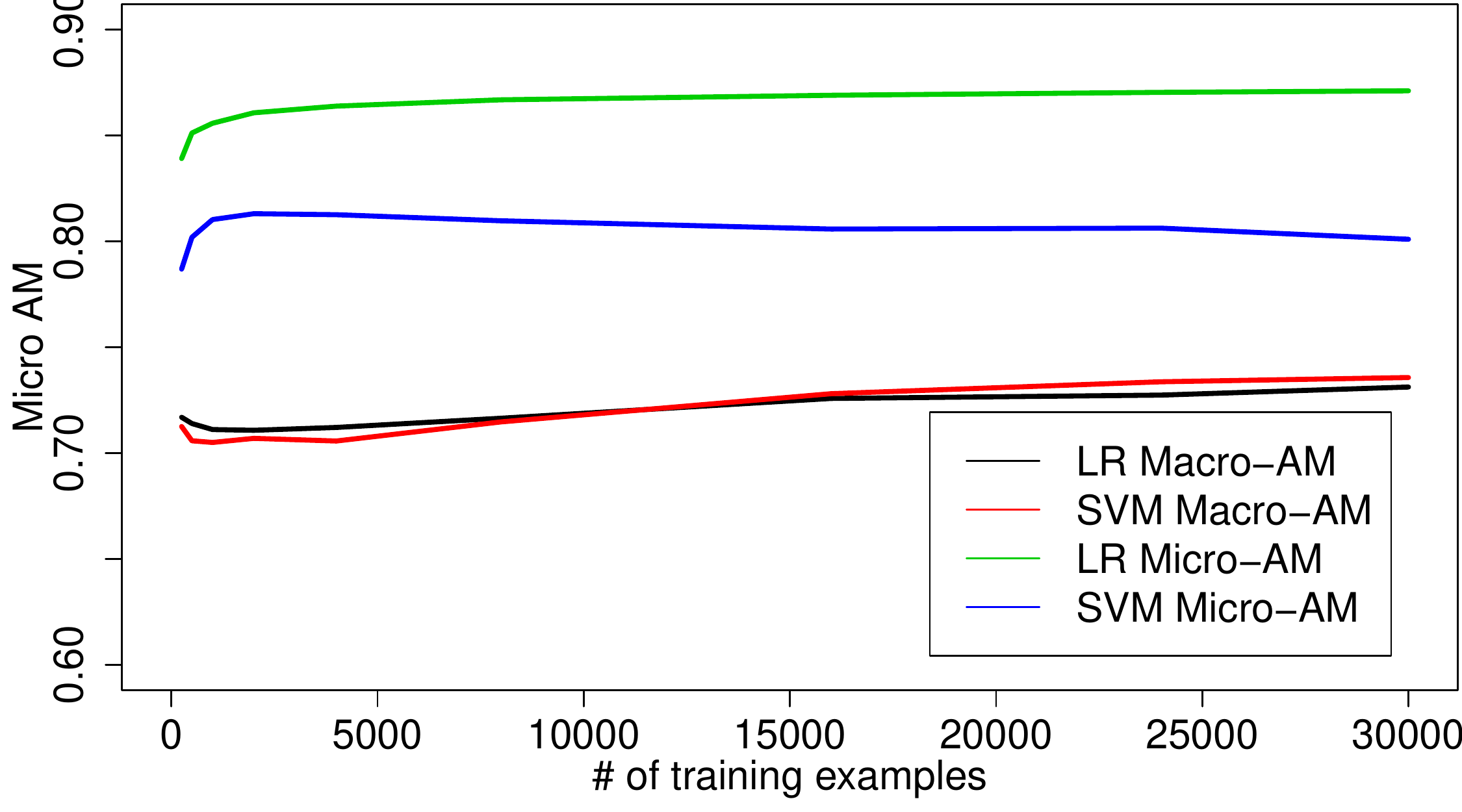} \\
\end{tabular}
\caption{Average test set performance on benchmark data sets for multi-label classification as a function of the number of training examples. Macro- and micro-averaged F-measure and AM are plotted for LR and SVM tuned for all the measures.}
\label{fig:multilabel_benchmark_results}
\end{center}
\end{figure*}
\begin{table}[t]
\begin{center}
\begin{tabular}{l r r r r}
\toprule
data set &  \# labels &  \# training examples & \# test examples & \#features \\
\midrule
\texttt{scene} & 6 & 1211 & 1169 & 294 \\ 
\texttt{yeast} & 14 & 1500 & 917 & 103 \\
\texttt{mediamill} & 101 & 30993 & 12914 & 120 \\
\bottomrule
\end{tabular}
\end{center}
\caption{Basic statistics for multi-label benchmark data sets}
\label{tbl:multilabel_data_sets}
%\vspace*{-1cm}
\end{table}

In the last experiment we use three multi-label benchmark data sets.\footnote{Datasets are taken from LibSVM repository: \url{http://www.csie.ntu.edu.tw/~cjlin/libsvmtools/datasets}} Table~\ref{tbl:multilabel_data_sets} provides a summary of basic statistics of these datasets.
The aim of the experiment is to verify the theoretical results in Section \ref{sec:multilabel} on learning the micro- and macro-averaged performance metrics. We use the F-measure and the AM-measure as in previous experiments. 

The data sets are already split into the training and testing parts.
As before we train a linear model using either SVM or LR on 2/3 of training examples. The rest of training data is used for tuning the threshold. For optimizing macro-averaged measures, we tune the threshold separately for each label. This approach agrees with our analysis given in Section~\ref{sec:macro-averaging}. For micro-averaging, we tune a common threshold for all labels: we simply collect predictions for all labels and find the best threshold using these values. This approach is justified by the theoretical analysis in Section~\ref{sec:micro-averaging}. Hence, the only difference between micro- and macro-versions of the algorithms is whether a single or multiple thresholds are tuned.
 In total we use 8 algorithms: two learning algorithms (LR/SVM), two performance measures (F/AM), and two types of averaging (Macro/Micro). 
 Note that our experiments include evaluating algorithms tuned for macro-averaging in terms of micro-averaged metrics, and vice versa. The goal of such cross-analysis is to determine the impact of threshold sharing for both averaging schemes.
As before, we incrementally increase the size of the training set and repeat training and threshold tuning 20 times (we use random draws of training instances into the proper training and the validation parts; the test set is always the same, as originally specified for each data set). The results are given in Fig~\ref{fig:multilabel_benchmark_results}.

The plots generally agree with the conclusions coming from the theoretical analysis, with some intriguing exceptions, however. As expected, LR tuned for a given performance metric gets the best result with respect to that metric in most of the cases. For the \texttt{scene} data set, however, the methods tuned for the micro-averaged metrics (single threshold shared among labels) outperform the ones tuned for macro-averaged metrics (separate thresholds tuned for each label), even when evaluated in terms of macro-averaged metrics. A similar result has been obtained by~\cite{Koyejo15}. It seems that tuning a single threshold shared among all labels can lead to a more stable solution that is less prone to overfitting, even though it is not the optimal thing to do for macro-averaged measures. We further report that, interestingly, SVM outperform LR in terms  of Macro-F on \texttt{mediamill} and this is the only case in which SVM get a better result than LR. 

\section{Summary}

\label{sec:summary}
We present a theoretical analysis of a two-step approach to optimize
classification performance metrics, which first learns a real-valued function $f$ on
a training sample by minimizing a surrogate loss, 
and then tunes the threshold on $f$ by optimizing the target performance metric on
a separate validation sample. We show that if the metric is a linear-fractional function,
and the surrogate loss is strongly proper composite, then 
the regret of the resulting classifier (obtained from thresholding real-valued $f$) 
measured with respect to the target metric is upperbounded by the regret of $f$ measured with respect
to the surrogate loss. 
The proof of our result goes by an intermediate bound of the regret with respect to
the target measure by a cost-sensitive classification regret.
As a byproduct, we get a bound on the cost-sensitive classification regret 
by a surrogate regret of a real-valued function which holds simultaneously for all misclassification costs. 
We also extend our results to cover multilabel classification and provide regret bounds
for micro- and macro-averaging measures. 
Our findings are backed up in a computational study on both synthetic and real data sets.

%To our knowledge, this is the first regret bound of this form applicable to generalized
%performance metrics in binary classification. 

%A natural question is whether our results can be generalized to other classification performance
%metrics, not necessarily of the 
%linear-fractional form (\ref{eq:pseudo_linear_function}). 
%In general, the answer is negative: it is known \citep{Narasimhan_etal2014} that for some performance
%metrics, $h^*_{\Psi}$ is \emph{not} necessarily a threshold function on $\eta$,
%and a bound given by Lemma \ref{lemma:main} cannot hold:
%In such case, 
%the bound given by  cannot hold without additional assumptions: 
%plugging $f=\Psi(\eta)$
%makes the right-hand side of the bound zero, while the left-hand side will remain
%nonzero for any value of the threshold. 
%On the other hand, 
%\citet{Narasimhan_etal2014} showed that $h^*_{\Psi}$ becomes a threshold function
%under some mild continuity assumptions on the distribution of $\eta(x)$. It remains
%an open question whether regret bound is possible in this case.
%Another open issue is the analysis of $\ell$-regret  and $\Psi$-regret on a restricted class of functions,
%e.g. linear models.

\bibliography{kotlowski}

\end{document}